\documentclass{article}

 \PassOptionsToPackage{numbers, compress}{natbib}

     \usepackage[preprint]{neurips_2020}

\usepackage[utf8]{inputenc} 
\usepackage[T1]{fontenc}    
\usepackage{hyperref}       
\usepackage{url}            
\usepackage{booktabs}       
\usepackage{amsfonts}       
\usepackage{nicefrac}       
\usepackage{microtype}      

\usepackage{lipsum}		
\usepackage{amsmath}
\usepackage{amsthm}
\usepackage{amssymb}
\usepackage{algpseudocode}
\usepackage{algorithm}
\usepackage{amsmath}
\usepackage{nicefrac}
\usepackage[toc,page]{appendix}
\newtheorem{theorem}{Theorem}
\newtheorem{lemma}{Lemma}

\newtheorem{proposition}{Proposition}

\usepackage{graphicx}
\usepackage{bm}
\graphicspath{ {images/} }
\usepackage{epstopdf}
\usepackage{comment}
\usepackage{color}

\title{\Large Bandit Samplers for Training Graph Neural Networks}

\newcommand*\samethanks[1][\value{footnote}]{\footnotemark[#1]}
\author{%
  Ziqi Liu\thanks{Equal Contribution.} \\
  Ant Financial Services Group \\
  \texttt{ziqiliu@antfin.com} \\
  \And
  Zhengwei Wu\samethanks \\
  Ant Financial Services Group \\
  \texttt{zejun.wzw@antfin.com} \\
  \AND
  Zhiqiang Zhang \\
  Ant Financial Services Group \\
  \texttt{lingyao.zzq@antfin.com} \\
  \And
  Jun Zhou \\
  Ant Financial Services Group \\
  \texttt{jun.zhoujun@antfin.com} \\
  \And
  Shuang Yang \\
  Ant Financial Services Group \\
  \texttt{shuang.yang@antfin.com} \\
  \And
  Le Song \\
  Ant Financial Services Group \\
  Georgia Institute of Technology \\
  \texttt{lsong@cc.gatech.edu} \\
  \And
  Yuan Qi \\
  Ant Financial Services Group \\
  \texttt{yuan.qi@antfin.com} 
}

\begin{document}

\maketitle

\begin{abstract}
Several sampling algorithms with variance reduction have been 
proposed for accelerating the training of Graph Convolution Networks (GCNs). 
However, due to the intractable computation of optimal sampling distribution,
these sampling algorithms are suboptimal for GCNs and are not
applicable to more general graph neural networks (GNNs) where 
the message aggregator contains learned weights rather than
fixed weights, such as Graph Attention Networks (GAT). 
The fundamental reason is that the embeddings of the neighbors or learned weights involved in the optimal sampling distribution
are \emph{changing} during the training and \emph{not known a priori}, 
but only \emph{partially observed} when sampled, thus 
making the derivation of an optimal variance reduced samplers non-trivial. 
In this paper, we formulate the optimization of the sampling 
variance as an adversary bandit problem, where the rewards are related to the node embeddings and learned weights, 
and can vary constantly. Thus a good sampler needs to acquire 
variance information about more neighbors (exploration) while at the same time optimizing the immediate sampling variance (exploit). We theoretically
show that our algorithm asymptotically approaches the optimal variance within a factor of 3. We show the efficiency and effectiveness of our approach on multiple
datasets.
\end{abstract}

\section{Introduction}
Graph neural networks~\cite{kipf2016semi,hamilton2017inductive} have emerged as a powerful tool for 
representation learning of 
graph data in irregular or non-euclidean 
domains~\citep{battaglia2018relational,wu2019comprehensive}. 
For instance, graph neural networks have demonstrated 
state-of-the-art performance on learning tasks
such as node classification, link and graph property 
prediction, with applications ranging from drug 
design~\cite{dai2016discriminative}, 
social networks~\cite{hamilton2017inductive}, 
transaction networks~\cite{liu2018heterogeneous}, 
gene expression networks~\cite{fout2017protein}, and 
knowledge graphs~\cite{schlichtkrull2018modeling}.

One major challenge of training GNNs comes from
the requirements of heavy floating point operations and 
large memory footprints, due to the recursive
expansions over the neighborhoods.
For a minibatch with a single vertex $v_i$, 
to compute its embedding $h_i^{(L)}$ at 
the $L$-th layer, we have to expand
its neighborhood from the $(L-1)$-th layer to
the $0$-th layer, i.e. $L$-hops neighbors. That will
soon cover a large portion of the
graph if particularly the graph is dense.
One basic idea of alleviating such ``neighbor explosion'' problem
was to sample neighbors in a top-down
manner, i.e. sample neighbors in the 
$l$-th layer given the nodes in the 
$(l+1)$-th layer recursively.

Several layer sampling 
approaches~\cite{hamilton2017inductive,chen2018fastgcn,huang2018adaptive,zou2019layer} 
have been proposed to alleviate above 
``neighbor explosion'' problem and improve the 
convergence of training GCNs, e.g. with importance
sampling. 
However, the optimal sampler~\cite{huang2018adaptive}, 
$q_{ij}^{\star} = \frac{\alpha_{ij}\|h_j^{(l)}\|^2}
{\sum_{k \in \mathcal{N}_i} \alpha_{ik} \|h_k^{(l)}\|^2}$ 
for vertex $v_i$,
to minimize the variance of the estimator $\hat{h}_i^{(l+1)}$ 
involves all its neighbors' hidden 
embeddings, i.e. $\{\hat{h}_j^{(l)} | v_j \in \mathcal{N}_i \}$, 
which is infeasible to be computed because we can only observe
them partially while doing sampling. Existing 
approaches~\cite{chen2018fastgcn,huang2018adaptive,zou2019layer}
typically compromise the optimal sampling distribution via approximations,
which may impede the convergence. Moreover, such approaches are 
not applicable to more general cases
where the weights or kernels $\alpha_{ij}$'s are not known a priori,
but are learned weights parameterized by 
attention functions~\cite{velivckovic2017graph}. 
That is, both the hidden embeddings and learned weights
involved in the optimal sampler constantly
\emph{vary} during the training process, and only \emph{part} of 
the unnormalized attention values or hidden embeddings
can be observed while do sampling.

{\bfseries Present work}. We derive novel variance reduced 
samplers for training of GCNs and attentive GNNs with a 
fundamentally different perspective. That is, different 
with existing approaches that need to compute the 
immediate sampling distribution, 
we maintain nonparametric estimates
of the sampler instead, and update the 
sampler towards optimal variance after we 
acquire partial knowledges 
about neighbors being sampled, as the algorithm
iterates.

To fulfil this purpose, we formulate the optimization
of the samplers as a bandit problem, where
the regret is the gap between expected
loss (negative reward) under current policy (sampler) and expected
loss with optimal policy. We define the reward 
with respect to each action, i.e. the choice 
of a set of neighbors with sample size $k$,
as the derivatives of the sampling variance, and show
the variance of our samplers asymptotically approaches
the optimal variance within a factor of $3$.
Under this problem formulation, we propose two bandit algorithms.
The first algorithm based on multi-armed bandit (MAB) chooses 
$k < K$ arms (neighbors) repeatedly.
Our second algorithm based on MAB with multiple plays
chooses a combinatorial set of neighbors with size $k$ only once.

To summarize,
(\textbf{1}) We recast the sampler for GNNs as a bandit 
problem from a fundamentally different perspective. It works
for GCNs and attentive GNNs while existing approaches apply
only to GCNs.
(\textbf{2}) We theoretically show that the regret with 
respect to the variance of our estimators asymptotically 
approximates the optimal sampler within a factor of 3 while
no existing approaches optimize the sampler. 
(\textbf{3}) We empirically show that our approachs are way competitive in
terms of convergence and sample variance, compared with 
state-of-the-art approaches on multiple public datasets.

\section{Problem Setting}

Let $\mathcal{G} = (\mathcal{V}, \mathcal{E})$ denote the graph with 
$N$ nodes $v_i \in \mathcal{V}$, and edges 
$(v_i, v_j) \in \mathcal{E}$. Let the adjacency matrix 
denote as $A \in \mathbb{R}^{N \times N}$. 
Assuming the feature matrix $H^{(0)} \in \mathbb{R}^{N\times D^{(0)}}$ 
with $h_i^{(0)}$ denoting the $D^{(0)}$-dimensional feature 
of node $v_i$. We focus on the following simple 
but general form of GNNs:
\begin{equation}\label{eq:gnn}
h_i^{(l+1)} = \sigma\Big(\sum_{j=1}^N\,\alpha(v_i, v_j)\,h_j^{(l)}\,W^{(l)}\Big),~~~~l = 0,\ldots, L-1
\end{equation}
where $h_i^{(l)}$ is the hidden embedding of node $v_i$ at the $l$-th layer,
$\pmb{\alpha} = \left(\alpha(v_i, v_j)\right) \in \mathbb{R}^{N \times N}$ is a kernel or weight matrix,  
$W^{(l)} \in \mathbb{R}^{D^{(l)} \times D^{(l+1)}}$ is the 
transform parameter on the $l$-th layer, and $\sigma(\cdot)$ 
is the activation function. The weight 
$\alpha(v_i, v_j)$, or $\alpha_{ij}$ for simplicity, 
is non-zero only if $v_j$ is in the $1$-hop 
neighborhood $\mathcal{N}_i$ of $v_i$. It varies with 
the aggregation 
functions~\cite{battaglia2018relational,wu2019comprehensive}. 
For example, 
    \textbf{(1)} GCNs~\cite{dai2016discriminative,kipf2016semi}~define fixed weights as $\pmb{\alpha} = \tilde{D}^{-1} \tilde{A}$ or $\pmb{\alpha} = \tilde{D}^{-\frac{1}{2}} \tilde{A} \tilde{D}^{-\frac{1}{2}}$ respectively, where $\tilde{A} = A + I$, and $\tilde{D}$ is the diagonal node degree matrix of $\tilde{A}$.
    \textbf{(2)} The attentive GNNs~\cite{velivckovic2017graph,liu2019geniepath} define a learned weight $\alpha(v_i, v_j)$ by attention functions:
$\alpha(v_i, v_j) = \frac{\tilde{\alpha}(v_i, v_j;\theta)}{\sum_{v_k \in \mathcal{N}_i} \tilde{\alpha}(v_i, v_k;\theta)}$, where the unnormalized
attentions
    $\tilde{\alpha}(v_i, v_j; \theta) = \exp(\mathrm{ReLU}(a^{T}[Wh_i\| Wh_j]))$,
are parameterized by $\theta = \left\{a, W\right\}$. Different from GCNs, the learned weights $\alpha_{ij} \propto \tilde{\alpha}_{ij}$ can be evaluated only given all the unnormalized weights in the neighborhood.

The basic idea of layer sampling 
approaches~\cite{hamilton2017inductive,chen2018fastgcn,huang2018adaptive,zou2019layer}
was to recast the evaluation of Eq.~\eqref{eq:gnn} as
\begin{align}\label{eq:gnn_expectation}
\hat{h}_i^{(l+1)} = \sigma\left(N(i)\,\mathbb{E}_{p_{ij}} \left[\hat{h}_j^{(l)}\right]\,W^{(l)}\right), 
\end{align}
where $p_{ij} \propto \alpha_{ij}$, and 
$N(i) = \sum_j \alpha_{ij}$. Hence
we can evaluate each node $v_i$ at the $(l+1)$-th layer,
using a Monte Carlo
estimator with sampled neighbors at the $l$-th layer.
Without loss of generality, we assume $p_{ij} = \alpha_{ij}$
and $N(i)=1$
that meet the setting of attentive GNNs
in the rest of this paper.
To further reduce the variance, let us
consider the following importance sampling
\begin{align}
\hat{h}_i^{(l+1)} = \sigma_{W^{(l)}}\big(\hat{\mu}_i^{(l)}\big) 
= \sigma_{W^{(l)}}
\Big(\mathbb{E}_{q_{ij}}\left[\frac{\alpha_{ij}}{q_{ij}}\hat{h}_j^{(l)}\right]\Big),
\end{align}
where we use $\sigma_{W^{(l)}}(\cdot)$ to include transform parameter $W^{(l)}$ into
the function $\sigma(\cdot)$ for conciseness.
As such, one can find an alternative sampling distribution 
$q_i = (q_{ij_1}, ..., q_{ij_{|\mathcal{N}_i|}})$ 
to reduce the variance of an estimator,
e.g. a Monte Carlo estimator $\hat{\mu}_i^{(l)} = 
\frac{1}{k}\sum_{s=1}^{k}\frac{\alpha_{ij_s}}{q_{ij_s}}\hat{h}_{j_s}^{(l)}$,
where $j_s \sim q_{i}$. 

Take expectation over $q_i$, we define the variance of 
$\hat{\mu}_i^{(l)}=\frac{\alpha_{ij_s}}{q_{ij_s}}\hat{h}_{j_s}^{(l)}$
at step $t$ and $(l+1)$-th layer to be:
\begin{align}
\label{eqn:pseudo-variance}
\mathbb{V}^{t}(q_i) = \mathbb{E}\Big[\Big\|\hat{\mu}_i^{(l)}(t) - \mu_i^{(l)}(t)\Big\|^2\Big]
=\mathbb{E}\Big[\Big\|\frac{\alpha_{ij_s}(t)}{q_{ij_s}}h_{j_s}^{(l)}(t) - \sum_{j\in \mathcal{N}_i}\alpha_{ij}(t)h_j^{(l)}(t)\Big\|^{2}\Big].
\end{align}
Note that $\alpha_{ij}$ and $h(v_j)$ that are
inferred during training may
vary over steps $t$'s. 
We will explicitly include step $t$ and layer $l$ only 
when it is necessary.
By expanding Eq.~\eqref{eqn:pseudo-variance} one can write 
$\mathbb{V}(q_i)$ as the difference of two terms. The first is a 
function of $q_i$, which we refer to as the 
\textit{effective variance}:
\begin{equation}
\label{eqn:effective-variance}
	\mathbb{V}_{e}(q_i) = \sum_{j \in \mathcal{N}_i}\frac{1}{q_{ij}}\alpha_{ij}^{2}\left\|h_j\right\|^2,
\end{equation}
while the second does not depend on $q_i$, and we denote it by 
$\mathbb{V}_{c} = \left\|\sum_{j\in\mathcal{N}_i}\alpha_{ij}h_j\right\|^2$.
The optimal sampling distribution~\cite{chen2018fastgcn,huang2018adaptive} at $(l+1)$-th layer for vertex $i$ that minimizes the variance is:
\begin{equation}\label{eq:optimal_q}
q_{ij}^{\star} = \frac{\alpha_{ij}\|h_j^{(l)}\|^2}
{\sum_{k \in \mathcal{N}_i} \alpha_{ik} \|h_k^{(l)}\|^2}.
\end{equation}
However, evaluating this sampling distribution is
infeasible because we cannot have all the knowledges
of neighbors' embeddings in the denominator of Eq.~\eqref{eq:optimal_q}. Moreover, the $\alpha_{ij}$'s
in attentive GNNs could also vary during the training
procedure. Existing layer sampling approaches based
on importance sampling just ignore the effects of
norm of embeddings and assume the $\alpha_{ij}$'s are fixed
during training. As a result, the sampling distribution is 
suboptimal and only applicable to GCNs where the weights
are fixed. Note that our derivation above follows the setting
of node-wise sampling approaches~\cite{hamilton2017inductive}, 
but the claim remains to hold for layer-wise 
sampling approaches~\cite{chen2018fastgcn,huang2018adaptive,zou2019layer}.

\section{Related Works}
We summarize three types of works for training graph neural networks.

First, several ``\textit{layer sampling}'' 
approaches~\cite{hamilton2017inductive,chen2018fastgcn,huang2018adaptive,zou2019layer} have been 
proposed to alleviate the ``neighbor explosion'' problems. 
Given a minibatch of labeled vertices at each 
iteration, such approaches sample neighbors layer by 
layer in a top-down manner.
Particularly, node-wise samplers~\cite{hamilton2017inductive} randomly sample
neighbors in the lower layer given each node in the upper layer, 
while layer-wise samplers~\cite{chen2018fastgcn,huang2018adaptive,zou2019layer} 
leverage importance sampling to sample neighbors in the lower
layer given all the nodes in upper layer with sample sizes
of each layer be independent of each other. 
Empirically, the layer-wise samplers work even 
worse~\cite{chen2017stochastic}
compared with node-wise samplers, and
one can set an appropriate sample size for each
layer to alleviate the growth issue of node-wise samplers.
In this paper, we focus on optimizing the variance
in the vein of layer sampling approaches.
Though the derivation of our bandit samplers follows 
the node-wise samplers, it can be extended
to layer-wise. We leave this extension 
as a future work.

Second, Chen et al.~\cite{chen2017stochastic} proposed 
a variance reduced estimator by maintaining 
historical embeddings of each vertices,
based on the assumption that the embeddings of 
a single vertex would be
close to its history. This estimator uses a simple 
random sampler and works efficient in practice at
the expense of requiring an extra storage that
is linear with number of nodes.

Third, two ``\textit{graph sampling}'' 
approaches~\cite{chiang2019cluster,zeng2019graphsaint} 
first cut the graph into partitions~\cite{chiang2019cluster} or 
sample into subgraphs~\cite{zeng2019graphsaint}, then they train 
models on those partitions or subgraphs in a batch 
mode~\cite{kipf2016semi}.
They show that the training time of each epoch may be much faster
compared with ``layer sampling'' approaches.
We summarize the drawbacks as follows. 
First, the partition of the original graph
could be sensitive to the training problem. Second, these
approaches assume that all the vertices in the graph
have labels, however, in practice 
only partial vertices may have labels~\cite{hu2019cash,liu2018heterogeneous}.

{\bfseries GNNs Architecture.} 
For readers who are interested in the works related to the architecture of GNNs, please refer to
the comprehensive survey~\cite{wu2019comprehensive}.
Existing sampling approaches works only on GCNs, but
not on more advanced architectures like GAT~\cite{velivckovic2017graph}.

\section{Variance Reduced Samplers as Bandit Problems}\label{sec:problem}
We formulate the optimization of
sampling variance as a 
bandit problem. Our basic idea is that instead of 
explicitly calculating the intractable
optimal sampling distribution in Eq.~\eqref{eq:optimal_q} 
at each iteration, 
we aim to optimize a sampler or \textbf{policy} $Q_i^t$ for each vertex
$i$ over the horizontal steps $1\leq t \leq T$, and make the
variance of the estimator following this sampler
asymptotically approach the optimum 
$Q_{i}^{\star}=\underset{Q_i}{\mathrm{argmin}}\sum_{t=1}^{T}\mathbb{V}_{e}^{t}(Q_i)$,
such that 
$\sum_{t=1}^{T}\mathbb{V}_e^t(Q_i^t) \leq c\sum_{t=1}^{T}\mathbb{V}_e^t(Q_i^{\star})$ 
for some constant $c>1$. 
Each \textbf{action} of policy $Q_i^t$ is a choice of any 
subset of neighbors 
$S_i \subset \mathcal{N}_i$ where $S_i \sim Q_i^t$.
We denote $Q_{i,S_i}(t)$ as the probability 
of the action that $v_i$ chooses $S_i$ at $t$.
The gap to be minimized between
effective variance and the oracle is
\begin{align}\label{eq:bandit_obj}
\mathbb{V}_e^t(Q_i^t) - \mathbb{V}_e^t(Q_i^{\star}) \leq \langle Q_i^t - Q_i^{\star}, 
\nabla_{Q_i^t}\mathbb{V}_e^t(Q_i^t) \rangle.
\end{align}
Note that the function 
$\mathbb{V}_e^t(Q_i^t)$ is convex w.r.t
$Q_i^t$, hence for $Q_i^t$ and $Q_i^{\star}$ we have the 
upper bound derived on right hand of Eq.~\eqref{eq:bandit_obj}.
We define this upper bound as \textbf{regret} at $t$,
which means the expected loss (negative
reward) with policy $Q_i^t$ minus the expected loss with
optimal policy $Q_i^{\star}$. 
Hence the \textbf{reward} 
w.r.t choosing $S_i$ at $t$ is the negative 
derivative of the effective variance
$r_{i,S_i}(t)=-\nabla_{Q_{i,S_i}(t)}\mathbb{V}_e^t(Q_i^t)$.

In the following, we adapt this bandit problem in the adversary bandit 
setting~\cite{auer2002nonstochastic} because the rewards
vary as the training proceeds and do not follow a priori
fixed distribution~\cite{burtini2015survey}. We leave the
studies of other bandits as a future work.
We show in section~\ref{sec:regret_analysis} that 
with this regret the variances of our estimators 
asymptotically approach the optimal 
variance within a factor of $3$.

Our samplers sample $1$-element subset of neighbors $k$ times  
or a $k$-element subset of neighbors once
from the alternative sampling distribution 
$q_i^t = (q_{ij_1}(t), ..., q_{ij_{|\mathcal{N}_i|}}(t))$
for each vertex $v_i$.
We instantiate above framework under two bandit settings.
\textbf{(1)} In the adversary
MAB setting~\cite{auer2002nonstochastic}, 
we define the sampler $Q_i^t$ as $q_i^t$, that 
samples exact an \textbf{arm} (neighbor) 
$v_{j_s} \subset \mathcal{N}_i$ from $q_i^t$. 
In this case the set $S_i$ is
the 1-element subset $\{v_{j_s}\}$. To have a sample size of $k$
neighbors, we repeat this action $k$ times.
After we collected $k$ rewards
$r_{ij_s}(t) = -\nabla_{q_{i,j_s}(t)}\mathbb{V}_e^t(q_i^t)$ 
we update $q_i^t$ by 
\textbf{EXP3}~\cite{auer2002nonstochastic}.
\textbf{(2)} In the adversary MAB with multiple plays 
setting~\cite{uchiya2010algorithms},
it uses an efficient $k$-combination sampler 
(\textbf{DepRound}~\cite{gandhi2006dependent})
$Q_i$ to sample any $k$-element subset 
$S_i\subset \mathcal{N}_i$ that satisfies
$\sum_{S_i:j\in S_i}Q_{i,S_i}=q_{ij}, \forall v_j\in \mathcal{N}_i$,
where $q_{ij}$ corresponds to the alternative probability
of sampling $v_j$. As such, it allows us
to select from a set of ${|\mathcal{N}_i| \choose k}$
distinct \textbf{subsets of arms}
from $|\mathcal{N}_i|$ arms at once.
The selection can be done in $O(|\mathcal{N}_i|)$. 
After we collected the reward 
$-\nabla_{Q_{i,S_i}(t)}\mathbb{V}_e^t(Q_{i}^t)$,
we update $q_i^t$ by \textbf{EXP3.M}~\cite{uchiya2010algorithms}.

{\bfseries Discussions.} We have to select a sample
size of $k$ neighbors in GNNs. Note that in the 
rigorous bandit setting, exact one action should 
be made and followed by updating
the policy. In adversary MAB, we do the selection $k$ times and 
update the policy, hence strictly speaking applying MAB to our
problem is not rigorous. Applying MAB with multiple plays 
to our problem is rigorous because it allows
us to sample $k$ neighbors at once and update
the rewards together. For readers who are interested in 
EXP3, EXP3.M and DepRound, please find them in 
Appendix~\ref{appendix:alg}.

\begin{algorithm}
\caption{Bandit Samplers for Training GNNs.}
\label{alg:train_gnn}
\begin{algorithmic}[1]
\Require step $T$, sample size $k$, number of layers $L$, node features $H^{(0)}$, adjacency matrix $A$.
\State \textbf{Initialize:} $q_{ij}(1)=
	1/\left|\mathcal{N}_i\right| \text{if}\; j\in\mathcal{N}_i \;
	\text{else}\;0
$, 
$w_{ij}(1)=
	1\; \text{if}\; j\in\mathcal{N}_i \;
	\text{else}\;0
$. 
\For {$t=1$ to $T$}
\State Read a minibatch of labeled vertices at layer $L$.
\State Use sampler $q_{i}^t$ or \textbf{DepRound}$(k, q_i^t)$ to sample neighbors top-down with sample size $k$.
\State Forward GNN model via estimators defined in Eq.~\eqref{eq:estimator1} or Proposition~\ref{proposition:estimator2}.
\State Backpropagation and update GNN model.
\For{each $v_i$ in the $1$-st layer} 
\State Collect $v_i$'s $k$ sampled neighbors
$v_j \in S_i^t$, and rewards $r_i^t=\left\{r_{ij}(t): v_j \in S_i^t\right\}$. 
\State Update $q_i^{t+1}$ and $w_i^{t+1}$ by \textbf{EXP3}$(q_i^t,w_i^t,r_i^t,S_i^t)$ 
or \textbf{EXP3.M}$(q_i^t,w_i^t,r_i^t,S_i^t)$.
\EndFor
\EndFor
\State \Return GNN model.
\end{algorithmic}
\end{algorithm}

\section{Algorithms}
The framework of our algorithm is: \textbf{(1)} 
use a sampler $Q_i^t$ to sample $k$ arms from the 
alternative sampling distribution $q_i^t$ for any vertex $v_i$,
\textbf{(2)} establish the \emph{unbiased 
estimator}, \textbf{(3)} do feedforward and 
backpropagation, and finally \textbf{(4)}
\emph{calculate the rewards} and \emph{update the alternative
sampling distribution}
with a proper bandit algorithm. We show this framework
in Algorithm~\ref{alg:train_gnn}. Note that
the variance w.r.t $q_i$ in Eq.~\eqref{eqn:pseudo-variance} 
is defined only at the $(l+1)$-th layer, hence
we should maintain multiple $q_i$'s at each layer.
In practice, we find that maintain a single $q_i$ 
and update it only using rewards from the $1$-st 
layer works well enough. The time complexity of our algorithm
is same with any node-wise 
approaches~\cite{hamilton2017inductive}. In addition, it requires
a storage in $O(|\mathcal{E}|)$ to maintain nonparametric
estimates $q_i$'s.

It remains to instantiate the estimators, variances
and rewards related to our two bandit settings.
We name our first algorithm \textbf{GNN-BS} under
adversary MAB setting, and the second \textbf{GNN-BS.M}
under adversary MAB with multiple plays setting.
We first assume the weights $\alpha_{ij}$'s are fixed,
then extend to attentive GNNs that
$\alpha_{ij}(t)$'s change.



\subsection{GNN-BS: Graph Neural Networks with Bandit Sampler}\label{sec:bs}

In this setting, we choose $1$ arm and 
repeat $k$ times. We have the
following Monte Carlo estimator 
\begin{equation}\label{eq:estimator1}
\hat{\mu}_i = 
\frac{1}{k}\sum_{s=1}^{k}\frac{\alpha_{ij_s}}{q_{ij_s}}\hat{h}_{j_s}, \; j_s \sim q_{i}.
\end{equation}
This yields the variance
$
\mathbb{V}(q_i) = \frac{1}{k}\,\mathbb{E}_{q_i}\left[\left\|\frac{\alpha_{ij_s}}{q_{ij_s}}h_{j_s} - \sum_{j\in \mathcal{N}_i}\alpha_{ij}h_j\right\|^{2}\right].
$
Following Eq.~\eqref{eqn:effective-variance} and Eq.~\eqref{eq:bandit_obj}, 
we have the reward of $v_i$ picking neighbor $v_j$ at step $t$ as 
\begin{equation}\label{eq:reward1}
r_{ij}(t)=-\nabla_{q_{ij}(t)}\mathbb{V}_e^t(q_i^t)
=\frac{\alpha_{ij}^2}{k\cdot q_{ij}(t)^2}\|h_j(t)\|^2.
\end{equation}

\subsection{GNN-BS.M: Graph Neural Networks with Multiple Plays Bandit Sampler}\label{sec:bsm}

Given a vertex $v_i$, an important property of DepRound is that
it satisfies 
$\sum_{S_i:j\in S_i}Q_{i,S_i}=q_{ij}, \forall v_j\in\mathcal{N}_i$, 
where $S_i \subset \mathcal{N}_i$ is any subset of size $k$.
We have the following unbiased estimator.

\begin{proposition}\label{proposition:estimator2}
$\hat{\mu}_i = \sum_{j_s\in S_i}\frac{\alpha_{ij_s}}{q_{ij_s}}h_{j_s}$ 
is the unbiased estimator of $\mu_i=\sum_{j\in\mathcal{N}_i}\alpha_{ij}h_j$ 
given that $S_i$ is sampled from $q_i$ using the DepRound 
sampler $Q_{i}$, where $S_i$ is 
the selected $k$-subset neighbors of vertex $i$.
\end{proposition}
The effective variance of this estimator is 
$\mathbb{V}_e(Q_i) = \sum_{S_i\subset \mathcal{N}_i} Q_{i,S_i}\|\sum_{j_s \in S_i} \frac{\alpha_{ij_s}}{q_{ij_s}}h_{j_s}\|^2$.
Since the derivative of this effective variance w.r.t
$Q_{i,S_i}$ does not factorize, we instead have 
the following approximated effective variance 
using Jensen's inequality.
\begin{proposition}\label{proposition:mp_var_bound}
The effective variance can be approximated by $\mathbb{V}_e(Q_i) \leq \sum_{j_s \in \mathcal{N}_i} \frac{\alpha_{ij_s}}{q_{ij_s}}\|h_{j_s}\|^2$.
\end{proposition}
\begin{proposition}\label{proposition:mp_derivative}
The negative derivative of the approximated effective variance $\sum_{j_s \in \mathcal{N}_i} \frac{\alpha_{ij_s}}{q_{ij_s}}\|h_{j_s}\|^2$ w.r.t $Q_{i,S_i}$,
i.e. the reward of $v_i$ choosing $S_i$ at $t$ is
$r_{i,S_i}(t) = \sum_{j_s \in S_i} \frac{\alpha_{ij_s}}{q_{ij_s}(t)^2}\|h_{j_s}(t)\|^2$.
\end{proposition}
Follow EXP3.M we use the reward w.r.t each arm as 
$r_{ij}(t) = \frac{\alpha_{ij}}{q_{ij}(t)^2}\|h_{j}(t)\|^2,
\forall j\in S_i$.
Our proofs rely on the property of DepRound introduced above.

\subsection{Extension to Attentive GNNs}

In this section, we extend our algorithms to
attentive GNNs.
The issue remained is that the attention value $\alpha_{ij}$ 
can not be evaluated with only sampled neighborhoods,
instead, we can only compute the unnormalized attentions $\tilde{\alpha}_{ij}$.
We define the adjusted feedback attention values as follows:
\begin{align}
\alpha_{ij}^{\prime} = \sum_{j\in S_i}q_{ij}\cdot\frac{\tilde{\alpha}_{ij}}{\sum_{j\in S_i}\tilde{\alpha}_{ij}},
\end{align}
where $\tilde{\alpha}_{ij}$'s are the unnormalized attention values that can be obviously
evaluated when we have sampled $(v_i, v_j)$.
We use $\sum_{j\in S_i} q_{ij}$ as a surrogate of 
$\frac{\sum_{j\in S_i}\tilde{\alpha}_{ij}}{\sum_{j\in \mathcal{N}_i}\tilde{\alpha}_{ij}}$
so that we can approximate the truth attention values $\alpha_{ij}$ by our
adjusted attention values $\alpha_{ij}^{\prime}$.

\section{Regret Analysis}\label{sec:regret_analysis}
As we described in section~\ref{sec:problem}, the regret is defined as $\langle Q_i^t - Q_i^{\star}, 
\nabla_{Q_i^t}\mathbb{V}_e^t(Q_i^t) \rangle$. By choosing the reward as the negative derivative of the effective variance, we have the following theorem that our bandit sampling algorithms asymptotically approximate the optimal variance within a factor of 3.

\begin{theorem}
Using Algorithm~\ref{alg:train_gnn} with $\eta=0.4$ 
and $\delta=\sqrt{\frac{(1-\eta)\eta^4k^5\ln(n/k)}{Tn^4}}$ 
to minimize the effective variance with respect 
to $\{Q_{i}^t\}_{1\leq t\leq T}$, we have 
\begin{equation}
\sum_{t=1}^T\mathbb{V}_e^t(Q_i^t) \leq 3\sum_{t=1}^T\mathbb{V}_e^t(Q_i^{\star}) + 10\sqrt{\frac{Tn^4\ln(n/k)}{k^3}}
\end{equation}
where $T\geq \ln(n/k)n^2(1-\eta)/(k\eta^2)$, $n=|\mathcal{N}_i|$.
\end{theorem}
Our proof follows \cite{salehi2017stochastic} by upper and lower bounding the potential function. The upper and lower bounds are the functions of the alternative sampling probability $q_{ij}(t)$ and the reward $r_{ij}(t)$ respectively. By multiplying the upper and lower bounds by the optimal sampling probability $q_{i}^{\star}$ and using the reward definition in \eqref{eq:reward1}, we have the upper bound of the effective variance.
The growth of this regret is sublinear in terms of $T$.
The regret decreases in polynomial as sample size $k$ grows.
Note that the number
of neighbors $n$ is always well bounded in pratical graphs,
and can be considered as a moderate constant number.
Compared with existing layer sampling 
approaches~\cite{hamilton2017inductive,chen2018fastgcn,zou2019layer}
that have a fixed variance given the specific estimators, 
this is the first work optimizing the sampling
variance of GNNs towards optimum.
We will empirically show the sampling variances in experiments.

\setlength{\tabcolsep}{1pt}
{\footnotesize
\begin{table*}
  \centering
  \caption{Dataset summary. ``s'' dontes multi-class task, and ``m'' denotes multi-label task.}
  \label{tb:data}
  \begin{tabular}{ccccccccc}
    Dateset & V & E & Degree &  \# Classes & \# Features & \# train & \# validation & \# test \\
    \midrule
    Cora&  $2,708$ & $5,429$ & $2$ & $7$ (s) & $1,433$ & $1,208$ & $500$ & $1,000$ \\
    Pubmed & $19,717$ & $44,338$ & $3$ & $3$ (s)  & $500$ & $18,217$ & $500$ & $1,000$ \\
    PPI &    $56,944$ & $818,716$ & $15$ & $121$ (m) & $50$ & $44,906$ & $6,514$ & $5,524$\\
    Reddit & $232,965$ & $11,606,919$ & $50$ & $41$ (s) & $602$ & $153,932$ & $23,699$ & $55,334$\\
    Flickr & $89,250$ & $899,756$ & $10$ & $7$ (s) & $500$ & $44,625$ & $22,312$ & $22,313$\\
  \bottomrule
\end{tabular}
\end{table*}
}

\section{Experiments}
In this section, we conduct extensive experiments compared with 
state-of-the-art approaches to show the advantage of our training
approaches. We use the following rule
to name our approaches: GNN architecture plus bandit sampler.
For example, \textbf{GCN-BS}, \textbf{GAT-BS} and 
\textbf{GP-BS} denote the training
approaches for GCN, GAT~\cite{velivckovic2017graph} 
and GeniePath~\cite{liu2019geniepath} respectively.

The major purpose of this paper is to compare the effects
of our samplers with existing training algorithms, so 
we compare them by training the same GNN architecture.
We use the following architectures unless otherwise stated.
We fix the number of layers as $2$ as 
in~\cite{kipf2016semi} for all comparison algorithms.
We set the dimension of hidden embeddings as $16$ for Cora and Pubmed,
and $256$ for PPI, Reddit and Flickr. For a fair comparison, we 
do not use the normalization layer~\cite{ba2016layer}
particularly used in some 
works~\cite{chen2017stochastic,zeng2019graphsaint}.
For attentive GNNs, we use the attention layer proposed 
in GAT. we set the number of multi-heads 
as $1$ for simplicity. 

We report results on $5$ benchmark data that include
\textit{Cora}~\cite{sen2008collective}, 
\textit{Pubmed}~\cite{sen2008collective},
\textit{PPI}~\cite{hamilton2017inductive},
\textit{Reddit}~\cite{hamilton2017inductive},
and \textit{Flickr}~\cite{zeng2019graphsaint}.
We follow the standard data splits,
and summarize the statistics in Table~\ref{tb:data}.

{\tiny
\begin{table*}[h]
\caption{Comparisons on the GCN architecture: testing Micro F1 scores.}
\label{tb:bench-gcn}
\begin{center}
\begin{tabular}{llllll}
\toprule
\textbf{Method}  &\textbf{Cora} &\textbf{Pubmed} &\textbf{PPI}& \textbf{Reddit}& \textbf{Flickr} \\
\midrule
GraphSAGE & $0.731 (\pm 0.014)$ &$0.890 (\pm 0.002)$ & $0.689(\pm 0.005)$ & $0.949 (\pm 0.001)$ & $0.494 (\pm 0.001)$ \\
FastGCN & $0.827 (\pm 0.001)$ &$0.895 (\pm 0.005)$ & $0.502(\pm 0.003)$ & $0.825 (\pm 0.006)$ & $0.500 (\pm 0.001)$ \\
LADIES & $0.843 (\pm 0.003)$ &$0.880 (\pm 0.006)$ & $0.574(\pm 0.003)$ & $0.932 (\pm 0.001)$ & $0.465 (\pm 0.007)$ \\
AS-GCN & $0.830 (\pm 0.001)$ &$0.888 (\pm 0.006)$ & $0.599(\pm 0.004)$ & $0.890 (\pm 0.013)$ & $0.506 (\pm 0.012)$ \\
S-GCN & $0.828 (\pm 0.001)$ &$0.893 (\pm 0.001)$ & $0.744(\pm 0.003)$ & $0.943 (\pm 0.001)$ & $0.501 (\pm 0.002)$ \\
ClusterGCN & $0.807 (\pm 0.006)$ &$0.887 (\pm 0.001)$ & $0.853(\pm 0.001)$ & $0.938 (\pm 0.002)$ & $0.418 (\pm 0.002)$ \\
GraphSAINT & $0.815 (\pm 0.012)$ &$0.899 (\pm 0.002)$ & $0.787(\pm 0.003)$ & $\bm{0.965} (\pm 0.001)$ & $0.507 (\pm 0.001)$ \\
\midrule
GCN-BS & $\bm{0.855} (\pm 0.005)$ &$\bm{0.903} (\pm 0.001)$ & $\bm{0.905}(\pm 0.003)$ & $0.957 (\pm 0.000)$ & $\bm{0.513} (\pm 0.001)$ \\
\bottomrule
\end{tabular}
\end{center}
\end{table*}
}

{\tiny
\begin{table*}[h]
\caption{Comparisons on the attentive GNNs architecture: testing Micro F1 scores.}
\label{tb:bench-gat}
\begin{center}
\begin{tabular}{llllll}
\toprule
\textbf{Method}  &\textbf{Cora} &\textbf{Pubmed} &\textbf{PPI}& \textbf{Reddit}& \textbf{Flickr} \\
\midrule
AS-GAT & $0.813 (\pm 0.001)$ &$0.884 (\pm 0.003)$ & $0.566(\pm 0.002)$ & NA & $0.472 (\pm 0.012)$ \\
GraphSAINT-GAT & $0.773 (\pm 0.036)$ &$0.886 (\pm 0.016)$ & $0.789(\pm 0.001)$ & $0.933 (\pm 0.012)$ & $0.470 (\pm 0.002)$ \\
\midrule
GAT-BS & $\bm{0.857} (\pm 0.003)$ &$\bm{0.894} (\pm 0.001)$ & $0.841(\pm 0.001)$ & $0.962 (\pm 0.001)$ & $\bm{0.513} (\pm 0.001)$ \\
GAT-BS.M & $\bm{0.857} (\pm 0.003)$ &$\bm{0.894} (\pm 0.000)$ & $0.867(\pm 0.003)$ & $0.962 (\pm 0.000)$ & $\bm{0.513} (\pm 0.001)$ \\
GP-BS & $0.811 (\pm 0.002)$ &$0.890 (\pm 0.003)$ & ${0.958}(\pm 0.001)$ & $\bm{0.964} (\pm 0.000)$ & $0.507 (\pm 0.000)$ \\
GP-BS.M & $0.811 (\pm 0.001)$ &$0.892 (\pm 0.001)$ & $\bm{0.965}(\pm 0.001)$ & $\bm{0.964} (\pm 0.000)$ & $0.507 (\pm 0.000)$ \\
\bottomrule
\end{tabular}
\end{center}
\end{table*}
\vspace{-0.1cm}
}

We summarize the comparison algorithms as follows.
\textbf{(1)} GraphSAGE~\cite{hamilton2017inductive} is
a node-wise layer sampling approach with a random 
sampler.
\textbf{(2)} FastGCN~\cite{chen2018fastgcn}, 
LADIES~\cite{zou2019layer}, and 
AS-GCN~\cite{huang2018adaptive} are
layer sampling approaches based on importance sampling.
\textbf{(3)} S-GCN~\cite{chen2017stochastic} can be 
viewed as an optimization solver for training of 
GCN based on a simply random sampler.
\textbf{(4)} ClusterGCN~\cite{chiang2019cluster} and
GraphSAINT~\cite{zeng2019graphsaint} are 
``graph sampling'' techniques that first partition 
or sample the graph into small subgraphs, then train 
each subgraph using the batch algorithm~\cite{kipf2016semi}.
\textbf{(5)} The open source algorithms that support the training
of attentive GNNs are AS-GCN and GraphSAINT. We denote
them as AS-GAT and GraphSAINT-GAT.

We do grid search for the following hyperparameters 
in each algorithm, i.e., the learning rate $\{0.01, 0.001\}$,
the penalty weight on the $\ell_2$-norm regularizers 
$\{0, 0.0001, 0.0005, 0.001\}$,
the dropout rate $\{0, 0.1, 0.2, 0.3\}$. By following the exsiting
implementations\footnote{Checkout: 
\url{https://github.com/matenure/FastGCN} or 
\url{https://github.com/huangwb/AS-GCN}}, 
we save the model based on the best
results on validation, and restore the model to report 
results on testing data in Section~\ref{sec:benchmark}.
For the sample size $k$ in GraphSAGE, S-GCN and our algorithms, 
we set $1$ for Cora and Pubmed,
$5$ for Flickr, $10$ for PPI and reddit.
We set the sample size in the first and second layer
for FastGCN and AS-GCN/AS-GAT as 
256 and 256 for Cora and Pubmed,
$1,900$ and $3,800$ for PPI, $780$ and $1,560$
for Flickr, and $2,350$ and $4,700$ for Reddit.
We set the batch size of all the layer sampling approaches
and S-GCN as $256$ for all the datasets. For ClusterGCN, 
we set the partitions
according to the suggestions~\cite{chiang2019cluster} 
for PPI and Reddit. We set the number of partitions 
for Cora and Pubmed as 10, for flickr as 200
by doing grid search. We set the architecture of GraphSAINT as 
``0-1-1''\footnote{Checkout \url{https://github.com/GraphSAINT/} 
for more details.} which means MLP layer followed by two 
graph convolution layers. We use the ``rw'' sampling 
strategy that reported as the best in their original paper
to perform the graph sampling procedure. We set the number 
of root and walk length as the paper suggested.

\begin{figure*}[h]
\includegraphics[width=0.33\textwidth]{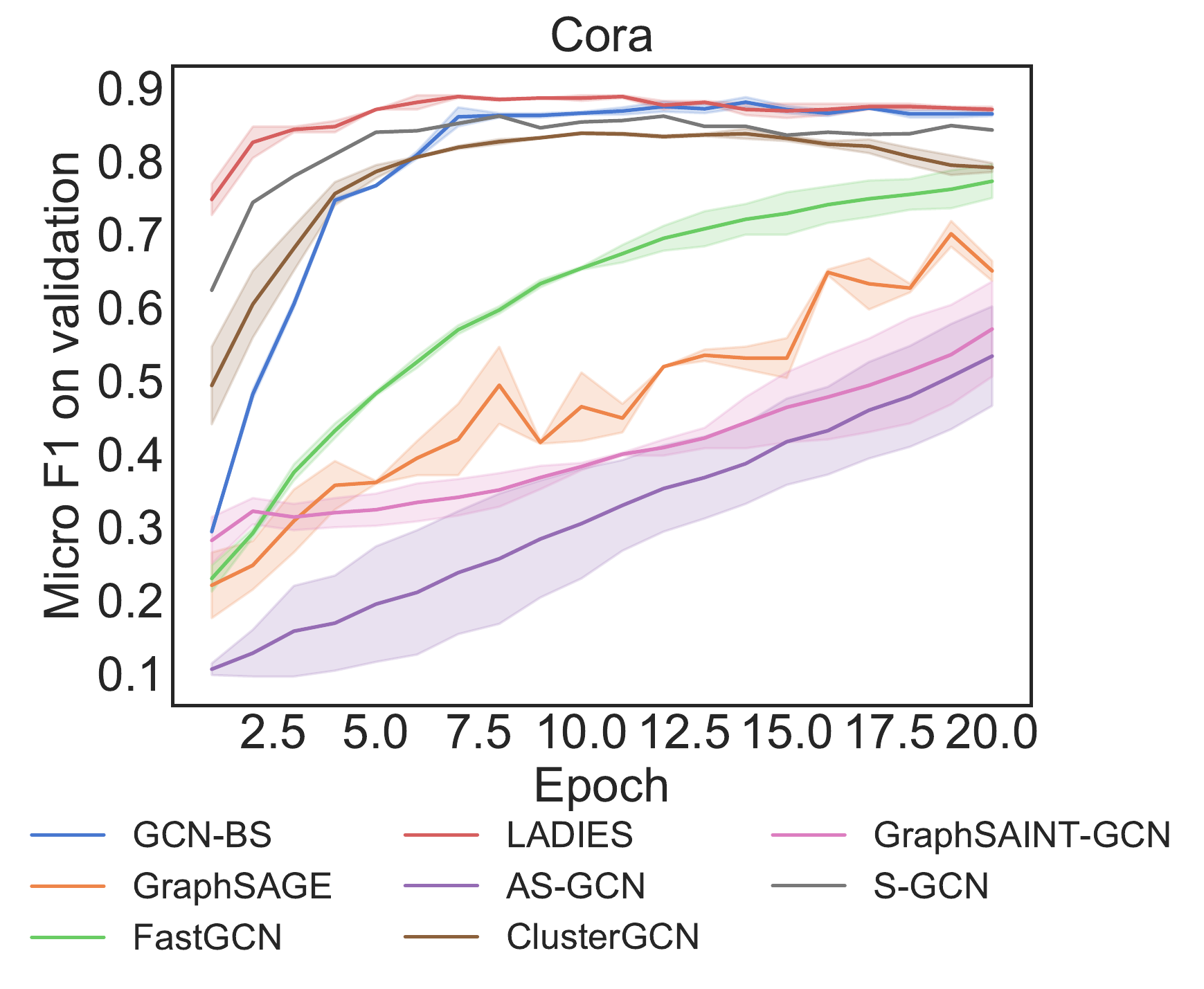}
\includegraphics[width=0.33\textwidth]{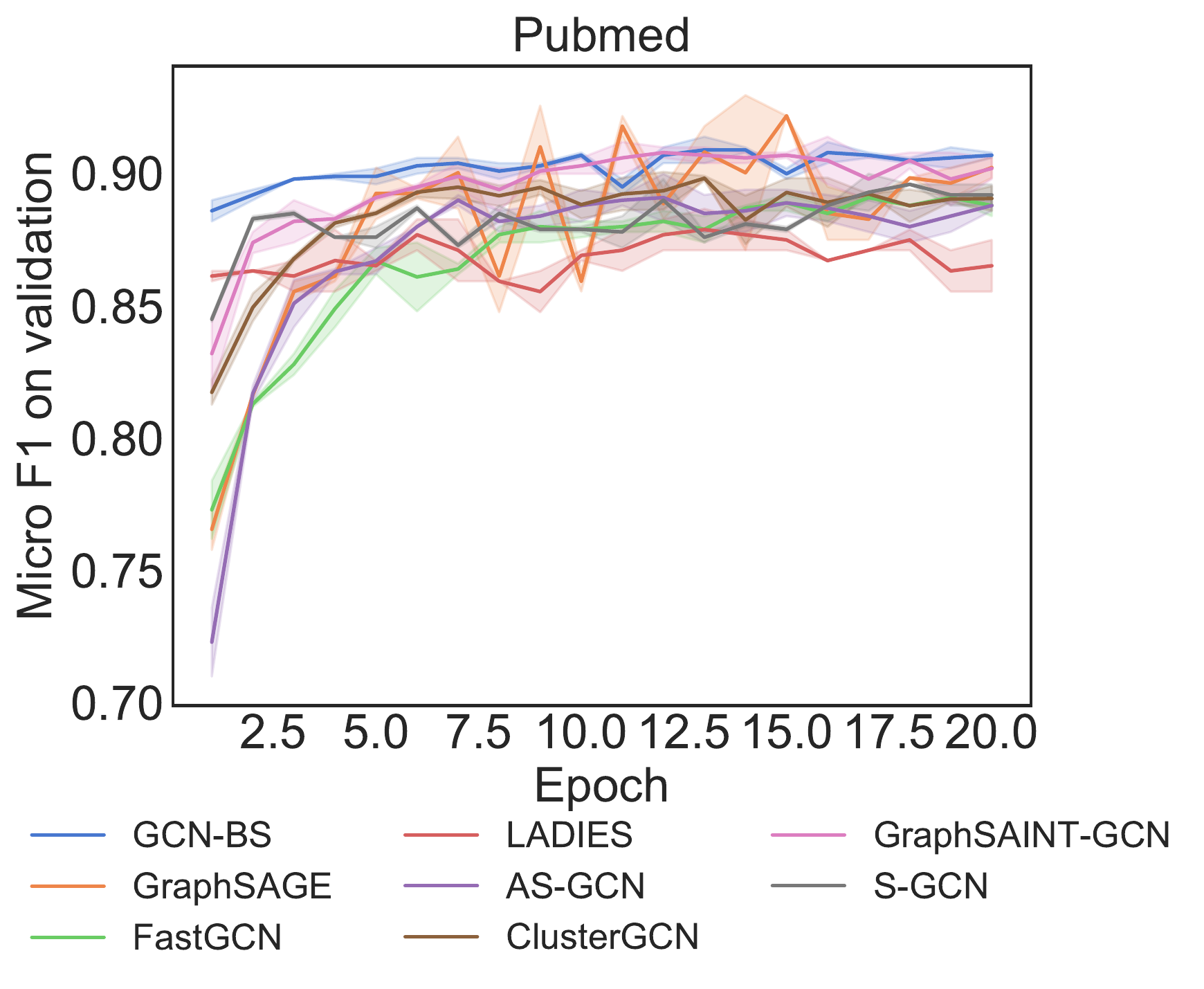}
\includegraphics[width=0.33\textwidth]{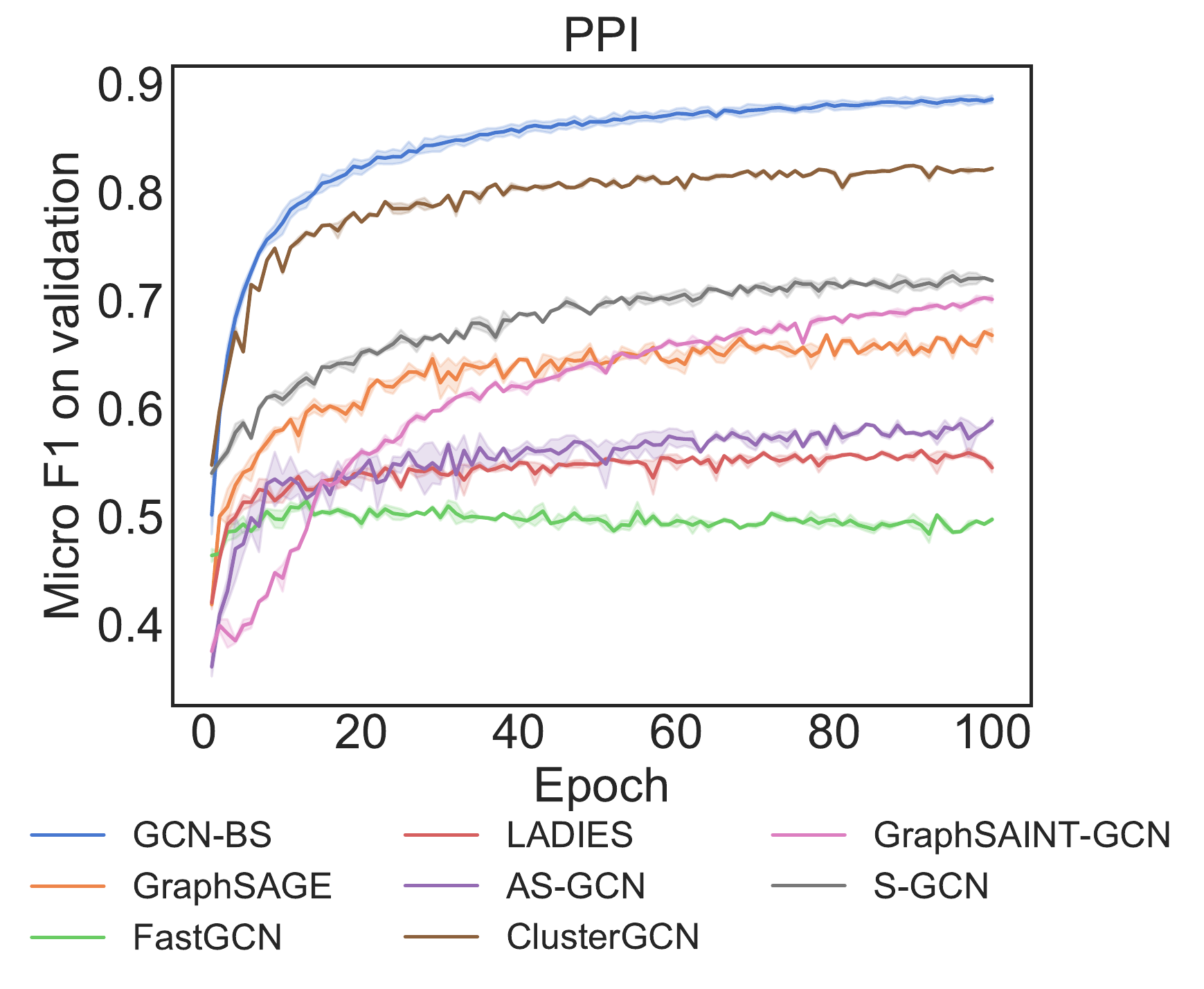}
\includegraphics[width=0.33\textwidth]{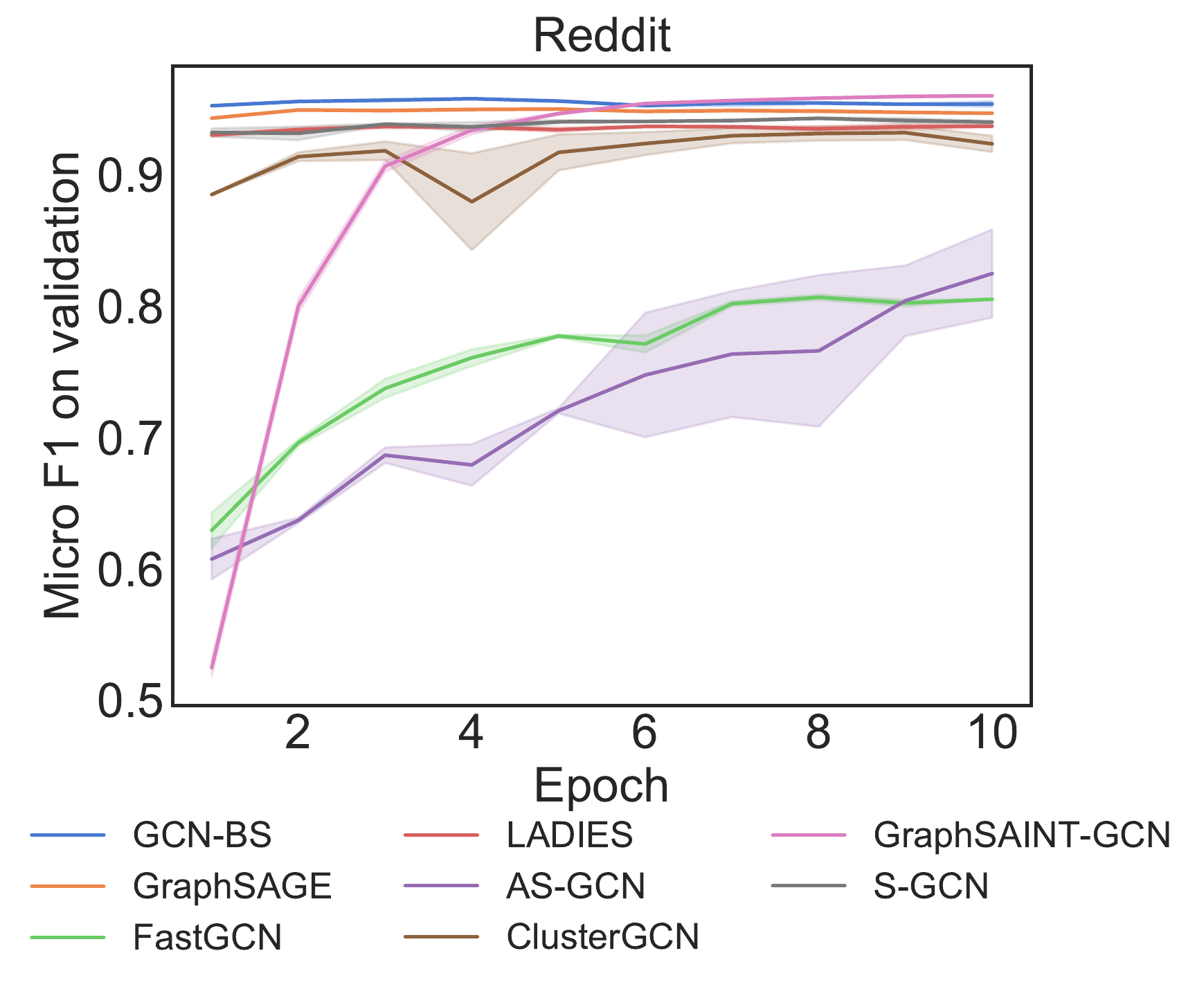}
\includegraphics[width=0.33\textwidth]{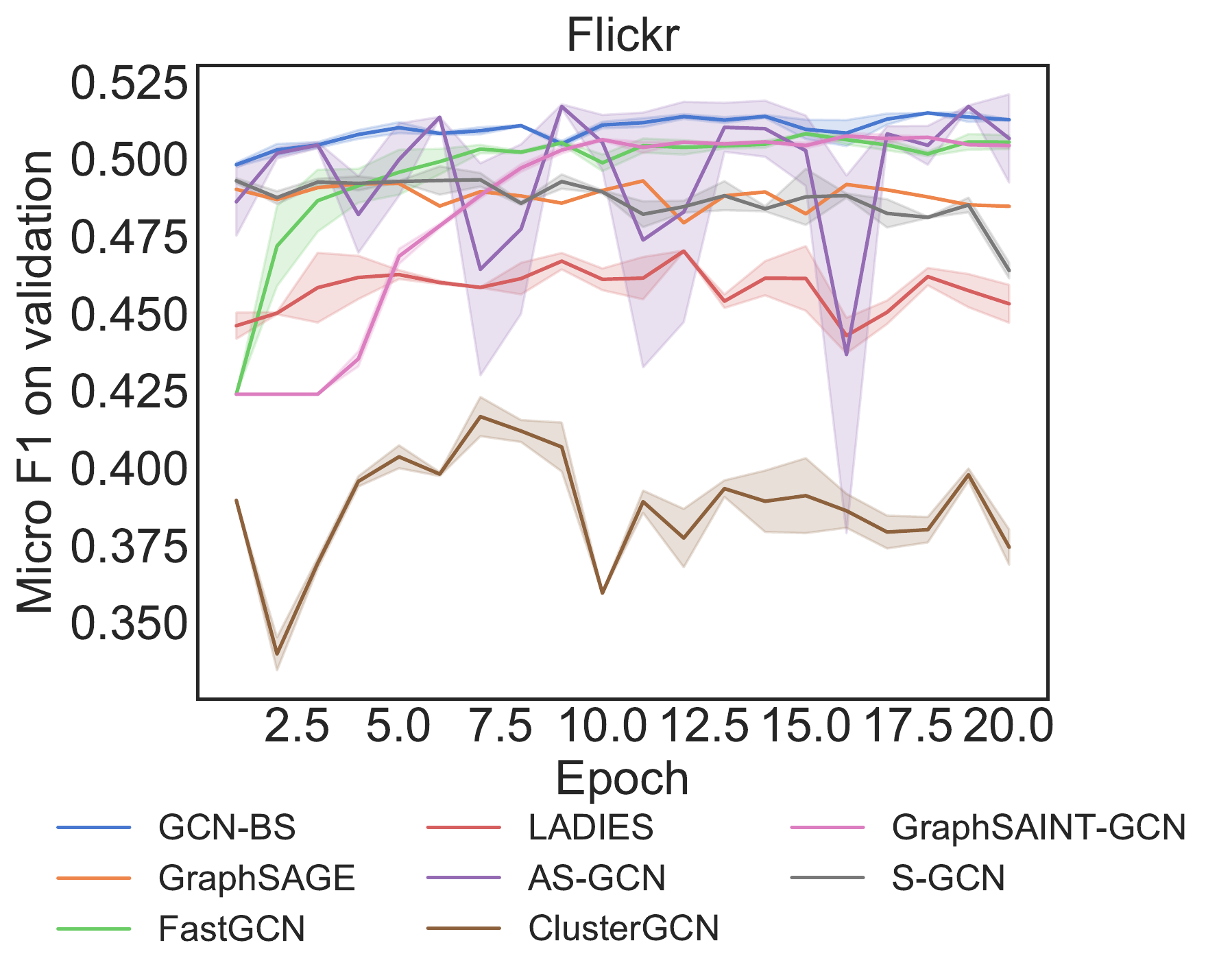}
\includegraphics[width=0.33\textwidth]{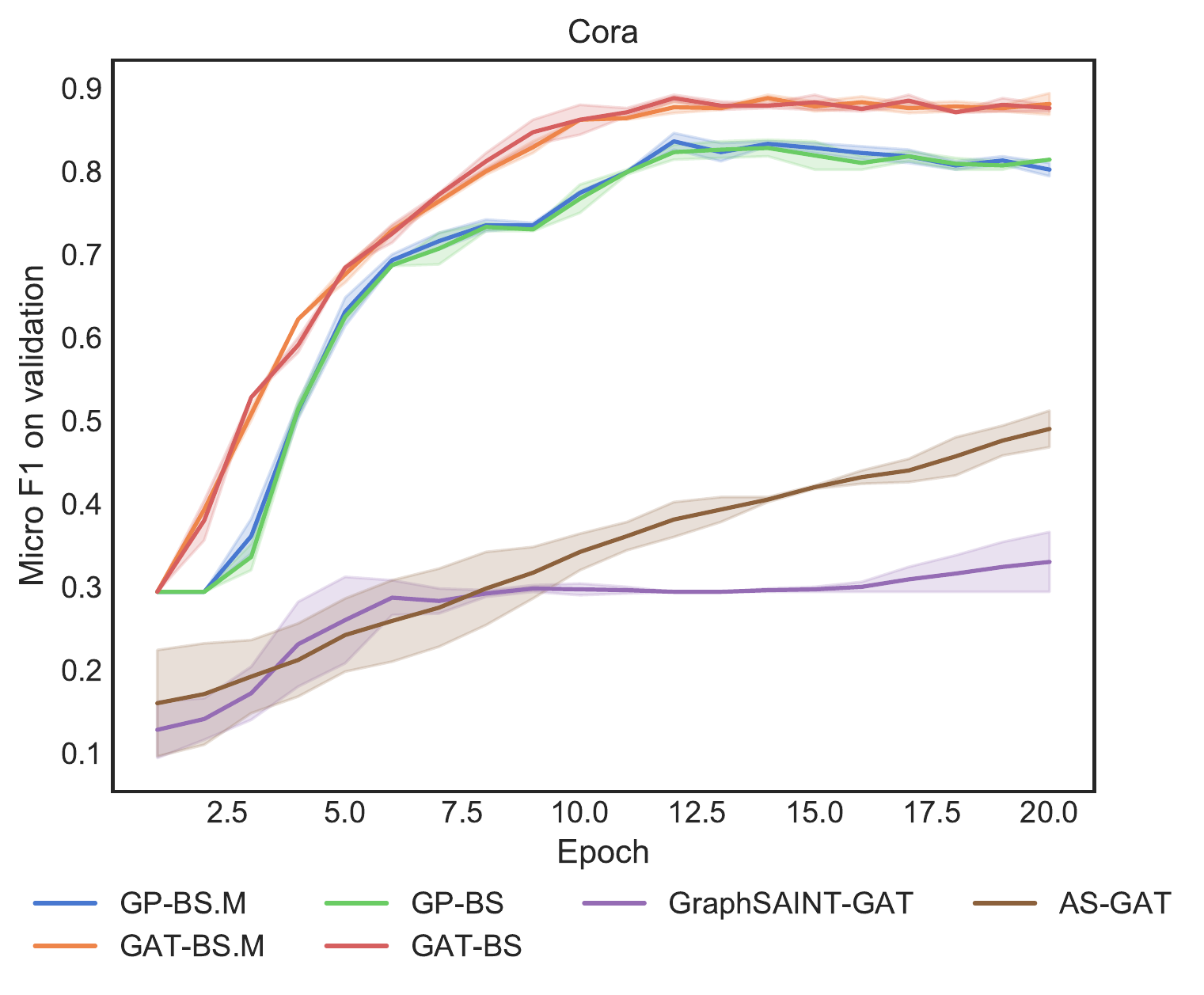}
\includegraphics[width=0.33\textwidth]{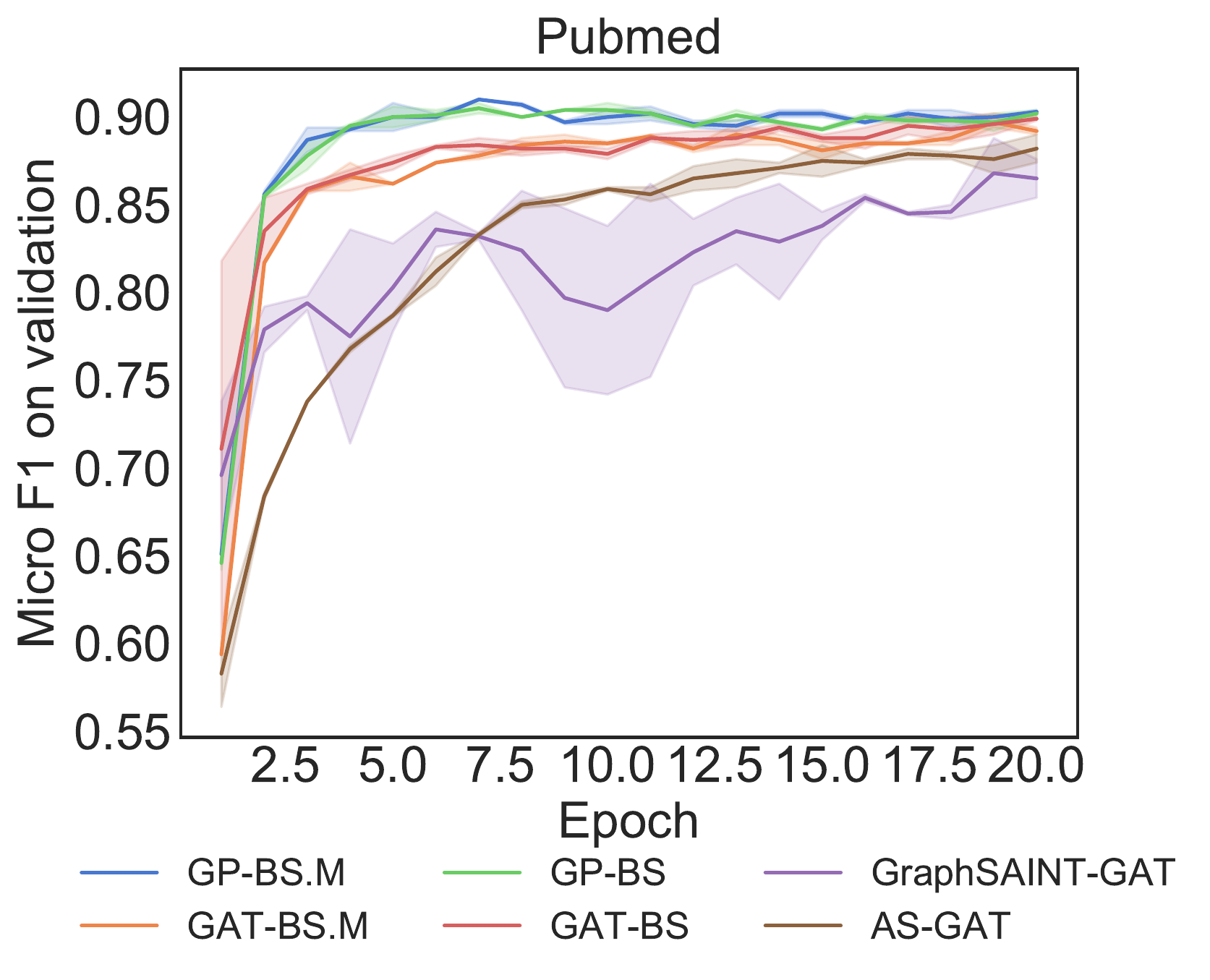}
\includegraphics[width=0.33\textwidth]{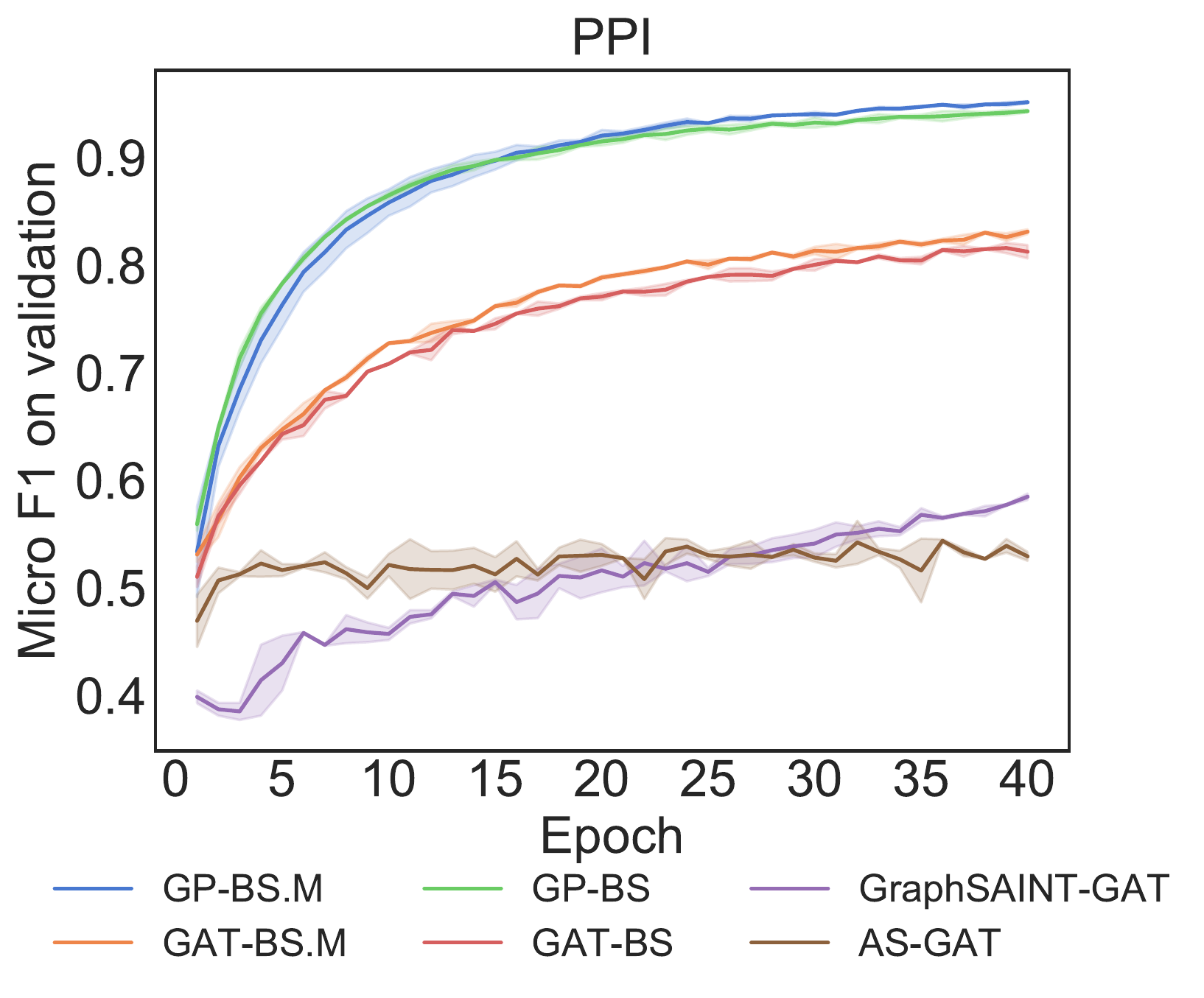}
\includegraphics[width=0.33\textwidth]{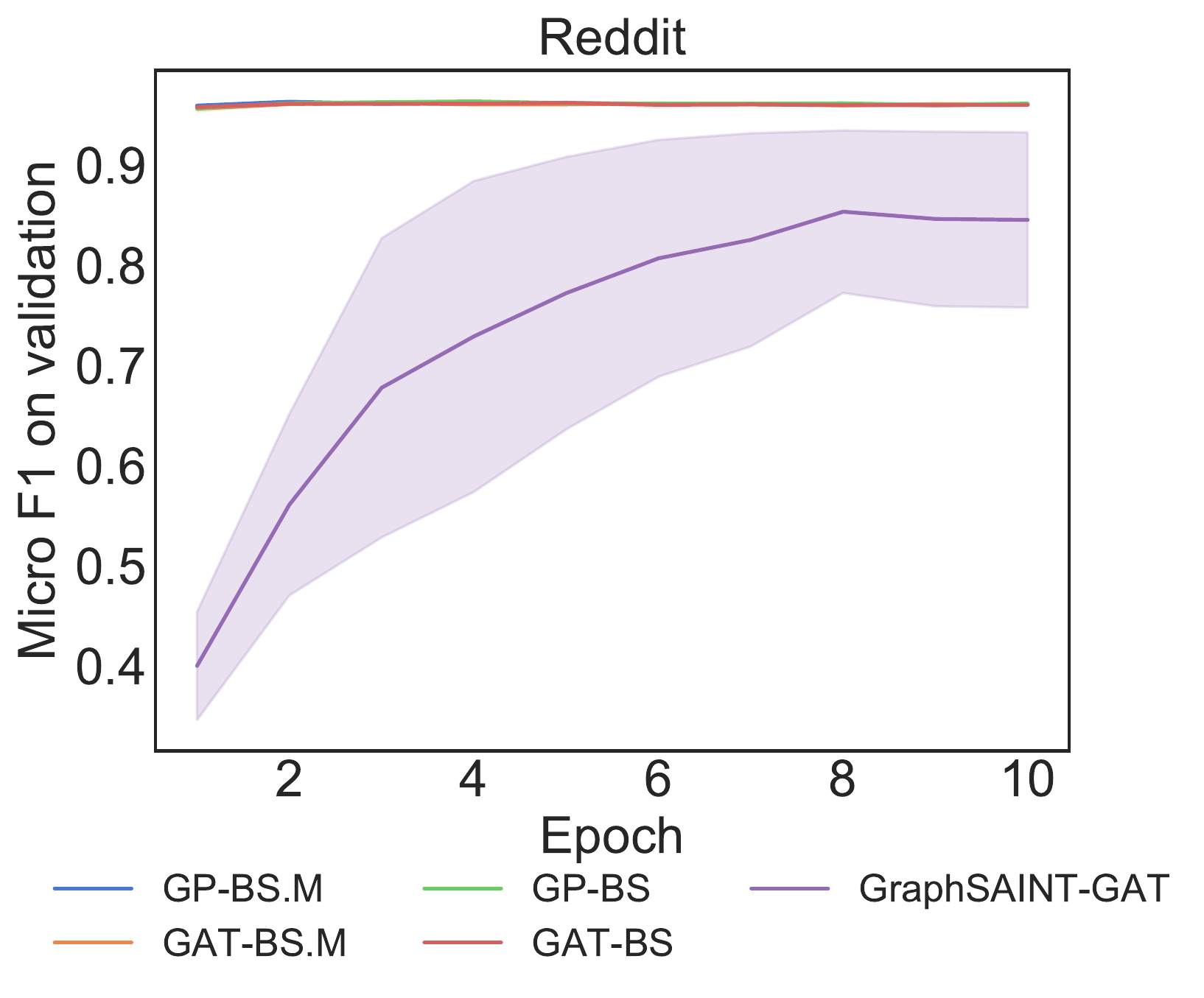}
\vspace{-0.2cm}
\caption{The convergence on validation in terms of epochs.}
\label{fig:convergence_gcn}
\end{figure*}
\vspace{-0.2cm}

\subsection{Results on Benchmark Data}\label{sec:benchmark}
We report the testing results on GCN and attentive
GNN architectures in Table~\ref{tb:bench-gcn} and 
Table~\ref{tb:bench-gat} respectively.
We run the results of each algorithm $3$ times and
report the mean and standard deviation.
The results on the two layer GCN architecture show that our
GCN-BS performs the best on most of datasets. 
The results on the two layer attentive GNN architecture show
the superiority of our algorithms on training more complex
GNN architectures. GraphSAINT or AS-GAT cannot compute 
the softmax of learned weights, but
simply use the unnormalized weights to perform the aggregation.
As a result, most of results from AS-GAT and GraphSAINT-GAT
in Table~\ref{tb:bench-gat} are worse than their results
in Table~\ref{tb:bench-gcn}. Thanks to the power of 
attentive structures in GNNs, our algorithms perform
better results on PPI and Reddit compared with
GCN-BS, and significantly outperform the results from 
AS-GAT and GraphSAINT-GAT.


\subsection{Convergence}
In this section, we analyze the convergences of 
comparison algorithms on the two layer GCN and 
attentive GNN architectures in Figure~\ref{fig:convergence_gcn} 
in terms of epoch.
We run all the algorithms $3$ times and show the mean and standard 
deviation.
Our approaches consistently converge to better results 
with faster rates and lower variances in most of
datasets like Pubmed, PPI, Reddit and Flickr compared with
the state-of-the-art approaches.
The GNN-BS algorithms perform very similar to GNN-BS.M, even though
strictly speaking GNN-BS does not follow the rigorous MAB setting.
Furthermore, we show a huge improvement on the training
of attentive GNN architectures compared with GraphSAINT-GAT
and AS-GAT.
The convergences on validation in terms of 
timing (seconds), compared with layer sampling
approaches, in Appendix~\ref{appendix:convergences} 
show the similar results. 
We further give a discussion about timing
among layer sampling approaches and graph sampling approaches
in Appendix~\ref{appendix:layer_vs_graph}.

\begin{figure}[h]
\includegraphics[width=0.49\textwidth]{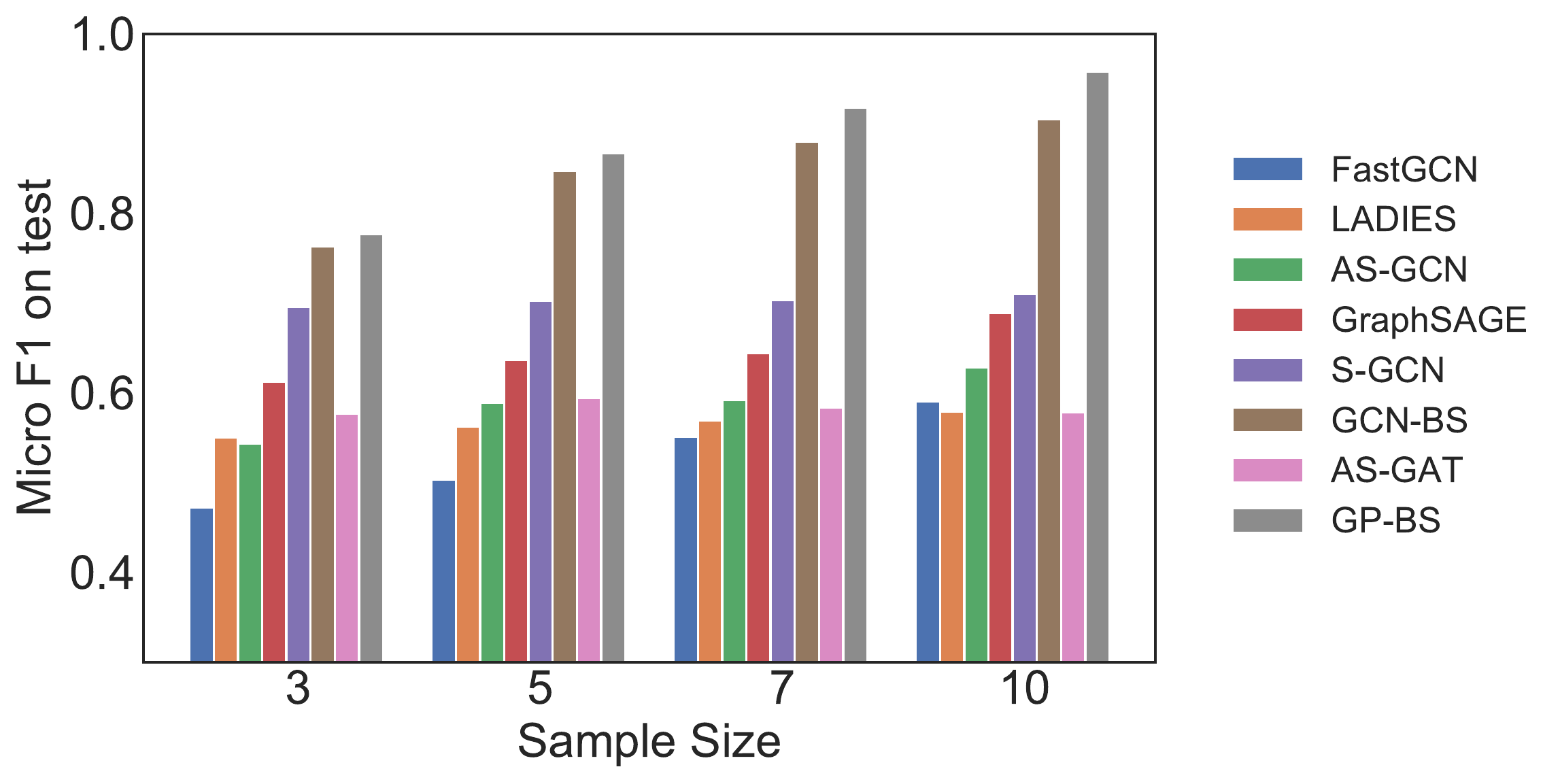}
\includegraphics[width=0.49\textwidth]{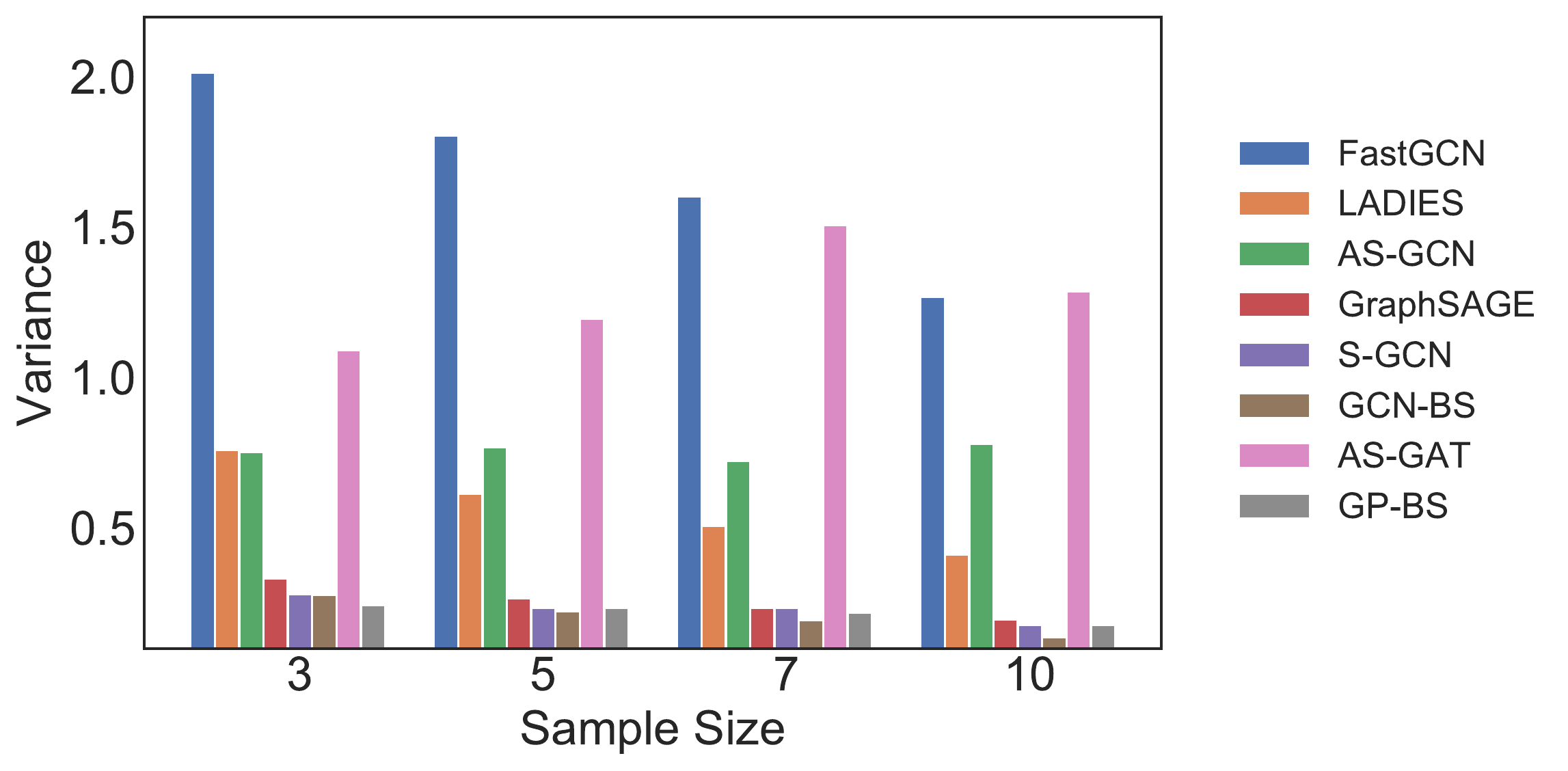}
\caption{Comparisons on PPI by varying the sample sizes: 
(\textbf{left}) F1 score, (\textbf{right}) sample variances.}
\label{fig:sample_size}
\end{figure}

\subsection{Sample Size Analysis}
We analyze the sampling variances and accuracy
as sample size varies using PPI data.
Note that existing layer sampling approaches do not
optimize the variances once the 
samplers are specified. As a result, their
variances are simply fixed~\cite{zou2019layer}, while
our approaches asymptotically appoach the optimum.
For comparison, we train our models until convergence,
then compute the average sampling variances. 
We show the results in Figure~\ref{fig:sample_size}.
The results are grouped into two categories, i.e.
results for GCNs and attentive GNNs respectively.
The sampling variances of our approaches are smaller 
in each group, and even be smaller than the 
variances of S-GCN that leverages a variance reduction solver. 
This explains the
performances of our approaches on testing Micro F1 scores.
We also find that the overall sampling variances of node-wise
approaches are way better than those of layer-wise
approaches.

%
%
%

\section{Conclusions}
In this paper, we show that the optimal layer samplers
based on importance sampling for training general 
graph neural networks are computationally intractable,
since it needs all the neighbors' hidden embeddings 
or learned weights.
Instead, we re-formulate the sampling problem as a bandit problem
that requires only partial knowledges from neighbors being sampled.
We propose two algorithms based on multi-armed bandit and MAB with
multiple plays, and show the variance of our bandit sampler 
asymptotically approaches the optimum within a factor of $3$.
Furthermore, our algorithms are not only applicable to 
GCNs but more general architectures like attentive GNNs.
We empirically show that our algorithms can converge to 
better results with faster rates and lower variances
compared with state-of-the-art approaches.


%

%
%


\clearpage
\appendix

\section{Algorithms}\label{appendix:alg}
\begin{algorithm}
\caption{$\mathrm{EXP3}(q_i^t,w_i^t,r_i^t,S_i^t)$.}
\label{alg:exp3}
\begin{algorithmic}[1]
\Require $\eta=0.4$, sample size $k$, neighbor size $n = |\mathcal{N}_i|$, $\delta=\sqrt{(1-\eta)\eta^4k^5\ln(n/k)/(Tn^4)}$.
\State Set 
	$$
	\hat{r}_{ij}(t)=
		r_{ij}(t)/q_{ij}(t)\; \text{if}\; j\in S_i^t \;\text{else}\;0
	$$
	$$
	w_{ij}(t+1)= w_{ij}(t)\exp(\delta\,\hat{r}_{ij}(t)/n) 
	$$
\State Set
$
q_{ij}(t+1) \gets (1-\eta) \frac{w_{ij}(t+1)}{\sum_{j\in \mathcal{N}_i} w_{ij}(t+1)} + 
\frac{\eta}{n}, \,\,\,\,\,\text{for}\,\, j \in \mathcal{N}_i
$
\end{algorithmic}
\end{algorithm}

\begin{algorithm}[ht]
\caption{EXP3.M$(q_i^t, w_i^t, r_i^t, S_i^t)$}
\label{alg:exp3m}
\begin{algorithmic}[1]
\Require $\eta=0.4$, sample size $k$, neighbor size $n=\left|\mathcal{N}_i\right|$, $\delta=\sqrt{(1-\eta)\eta^4 k^5 \ln(n/k)/(T n^4)}$, $U_i^t=\emptyset$.
	\State For $j\in\mathcal{N}_i$ set
	$$
	\hat{r}_{ij}(t)=
	\begin{cases}
		r_{ij}(t)/q_{ij}(t)& \text{if}\; j\in S_i^t \\
		0& \text{otherwise}
	\end{cases}
	$$
	$$
	w_{ij}(t+1)=
	\begin{cases}
		w_{ij}(t)\exp(\delta\hat{r}_{ij}(t)/n)& \text{if}\; j\notin U_i^t \\
		w_{ij}(t)& \text{otherwise}
	\end{cases}
	$$
	\If {$\mathop{\max}_{j\in\mathcal{N}_i}w_{ij}(t+1)\geq (\frac{1}{k}-\frac{\eta}{n})\sum_{j\in\mathcal{N}_i}{w_{ij}(t+1)/(1-\eta)}$}
		\State Decide $a_t$ so as to satisfy
		\begin{align*}
			\frac{a_t}{\sum_{w_{ij}(t+1)\geq a_t}a_t+\sum_{w_{ij}(t+1)<a_t}w_{ij}(t+1)} 
		    =(\frac{1}{k}-\frac{\eta}{n})/(1-\eta)
		\end{align*}
		\State Set $U_i^{t+1}=\{j:w_{ij}(t+1)\geq a_t\}$
                \Else
		\State Set $U_i^{t+1}=\emptyset$
                \EndIf
	\State Set
		$
		w_{ij}^{\prime}(t+1)=
		\begin{cases}
			w_{ij}(t+1)& \text{if}\; j\in\mathcal{N}_i\backslash U_i^{t+1} \\
			a_t& \text{if}\; j\in U_i^{t}
		\end{cases}
		$
	\State Set
		$
		q_{ij}(t+1)=k\left((1-\eta)\frac{w_{ij}^{\prime}(t+1)}{\sum_{j\in\mathcal{N}_i}w_{ij}^{\prime}(t+1)} + \frac{\eta}{n}\right)
		$ for
		$
		j\in\mathcal{N}_i
		$
\end{algorithmic}
\end{algorithm}

\begin{algorithm}
\caption{DepRound$(k,(q_{1},q_{2},...,q_{K}))$}
\label{alg:dep_round}
\begin{algorithmic}[1]
\State \textbf{Input:} Sample size $k(k<K)$, sample distribution $(q_1,q_2,...,q_K)$ with $\sum_{i=1}^{K}q_i=k$
\State \textbf{Output:} Subset of $[K]$ with $k$ elements
\While{there is an $i$ with $0<q_i<1$}
    \State Choose distinct $i$ and $j$ with $0<q_i<1$ and $0<q_j<1$
    \State Set $\beta=\min\{1-q_i, q_j\}$ and $\gamma=\min\{q_i,1-q_j\}$ 
    \State Update $q_i$ and $q_j$ as
    $$
    (q_i,q_j)=\begin{cases}
		(q_i+\beta,q_j-\beta)\; \text{with probability}\; \frac{\gamma}{\beta+\gamma} \\
		(q_i-\gamma,q_j+\gamma)\; \text{with probability}\; \frac{\beta}{\beta+\gamma}
	\end{cases}
    $$
\EndWhile
\State \textbf{return} $\{i:q_i=1,1\leq i\leq K\}$
\end{algorithmic}
\end{algorithm}

\section{Proofs}\label{appendix:proof}
\setcounter{theorem}{0}
\setcounter{property}{0}
\setcounter{proposition}{0}

\begin{proposition}\label{proposition:estimator2}
$\hat{\mu}_i = \sum_{j_s\in S_i}\frac{\alpha_{ij_s}}{q_{ij_s}}h_{j_s}$ 
is the unbiased estimator of $\mu_i=\sum_{j\in\mathcal{N}_i}\alpha_{ij}h_j$ 
given that $S_i$ is sampled from $q_i$ using the DepRound 
sampler $Q_{i}$, where $S_i$ is 
the selected $k$-subset neighbors of vertex $i$.
\end{proposition}
\begin{proof}
Let us denote $Q_{i,S_i}$ as the probability of vertex 
$v_i$ choosing any $k$-element subset $S_i \subset \mathcal{N}_i$ 
from the $K$-element set $\mathcal{N}_i$ using DepRound sampler $Q_i$.
This sampler follows the alternative sampling distribution
$q_i = (q_{ij_{1}}, ..., q_{ij_{K}})$ where $q_{ij_{s}}$ denotes
the alternative probability of sampling neighbor $v_{j_s}$.
This sampler is guaranteed to satisfy $\sum_{S_i : j\in S_i}Q_{i,S_i} = q_{ij}$, i.e. the sum over the probabilities of all subsets $S_i$ that
contains element $j$ equals the probability $q_{ij}$.
\begin{align}
	\mathbb{E}\left[ \hat{\mu}_i \right]
	&= \mathbb{E}\left[\sum_{j_s\in S_i}\frac{\alpha_{ij_s}}{q_{ij_s}}h_{j_s}\right] \\ 
	&= \sum_{S_i\subset\mathcal{N}_i}Q_{i,S_i}\sum_{j_s\in S_i}\frac{\alpha_{ij_s}}{q_{ij_s}}h_{j_s} \\
	&= \sum_{j\in\mathcal{N}_i}\sum_{S_i:j\in S_i}Q_{i,S_i}\frac{\alpha_{ij}}{q_{ij}}h_j \\
	&= \sum_{j\in\mathcal{N}_i}\frac{\alpha_{ij}}{q_{ij}}h_j\sum_{S_i:j\in S_i}Q_{i,S_i} \\
	&= \sum_{j\in\mathcal{N}_i}\frac{\alpha_{ij}}{q_{ij}}h_j q_{ij} \label{eq1} \\
	&= \sum_{j\in\mathcal{N}_i}\alpha_{ij}h_j
\end{align}
\end{proof}

\begin{proposition}\label{proposition:mp_var_bound}
The effective variance can be approximated by $\mathbb{V}_e(Q_i) \leq \sum_{j_s \in \mathcal{N}_i} \frac{\alpha_{ij_s}}{q_{ij_s}}\|h_{j_s}\|^2$.
\end{proposition}
\begin{proof}
The variance is
\begin{align*}
\mathbb{V}(Q_i) &= \mathbb{E}\left[\left\|\sum_{j_s\in S_i}\frac{\alpha_{i j_s}}{q_{ij_s}}h_{j_s} - \sum_{j\in\mathcal{N}_i}\alpha_{ij}h_{j} \right\|^2\right] \\
	&= \sum_{S_i\subset\mathcal{N}_i}Q_{i,S_i}\left\|\sum_{j_s\in S_i}\frac{\alpha_{ij_s}}{q_{ij_s}}h_{j_s}\right\|^2 - \left\|\sum_{j\in\mathcal{N}_i}\alpha_{ij} h_{j}\right\|^2.
\end{align*}

Therefore the effective variance has following upper bound:
\begin{align*}
\mathbb{V}_e(Q_{i}) &= \sum_{S_i\subset\mathcal{N}_i}Q_{i,S_i}\left\|\sum_{j_s\in S_i}\frac{\alpha_{ij_s}}{q_{ij_s}}h_{j_s}\right\|^2 \\
	&\leq \sum_{S_i\subset\mathcal{N}_i} Q_{i,S_i}\sum_{j_s\in S_i}\alpha_{ij_s}\left\|\frac{h_{j_s}}{q_{ij_s}}\right\|^2 \;(Jensen's\, Inequality) \\
	&= \sum_{j_s\in\mathcal{N}_i}\sum_{S_i:j_s\in S_i}Q_{i,S_i}\alpha_{ij_s}\left\|\frac{h_{j_s}}{q_{ij_s}}\right\|^2 \\
	&= \sum_{j_s\in\mathcal{N}_i}\frac{\alpha_{ij_s}}{q_{ij_s}^2}\|h_{j_s}\|^2\sum_{S_i:j_s\in S_i}Q_{i,S_i} \\
	&= \sum_{j_s\in\mathcal{N}_i}\frac{\alpha_{ij_s}}{q_{ij_s}}\|h_{j_s}\|^2
\end{align*}
\end{proof}

\begin{proposition}\label{proposition:mp_derivative}
The negative derivative of the approximated effective variance $\sum_{j_s \in \mathcal{N}_i} \frac{\alpha_{ij_s}}{q_{ij_s}}\|h_{j_s}\|^2$ w.r.t $Q_{i,S_i}$,
i.e. the reward of $v_i$ choosing $S_i$ at $t$, is
$r_{i,S_i}(t) = \sum_{j_s \in S_i} \frac{\alpha_{ij_s}}{q_{ij_s}(t)^2}\|h_{j_s}(t)\|^2$.
\end{proposition}
\begin{proof}
Define the upper bound as $\hat{\mathbb{V}}_e(Q_i)= \sum_{j_s\in\mathcal{N}_i}\frac{\alpha_{ij_s}}{q_{ij_s}}\|h_{j_s}\|^2$,
then its derivative is 
\begin{align*}
\nabla_{Q_{i,S_i}}\hat{\mathbb{V}}_e(Q_i) &= \nabla_{Q_{i,S_i}}\sum_{j_s\in\mathcal{N}_i}\frac{\alpha_{ij_s}}{q_{ij_s}}\|h_{j_s}\|^2 \\
	&= \nabla_{Q_{i,S_i}}\sum_{j_s\in\mathcal{N}_i}\frac{\alpha_{ij_s}}{\sum_{S_i':j_s\in S_i'}Q_{i,S_i'}}\|h_{j_s}\|^2 \\
	&= \nabla_{Q_{i,S_i}}\sum_{j_s\in S_i}\frac{\alpha_{ij_s}}{\sum_{S_i':j_s\in S_i'}Q_{i,S_i'}}\|h_{j_s}\|^2 \\
	&= -\sum_{j_s\in S_i}\frac{\alpha_{j_s}}{q_{ij_s}^2}\|h_{j_s}\|^2 \; (chain\, rule)
\end{align*}
\end{proof}

Before we give the proof of Theorem~\ref{theorem:bs},
we first prove the following Lemma~\ref{lemma:1} that
will be used later.
\begin{lemma}\label{lemma:1}
	For any real value constant $\eta\leq 1$ and any valid distributions $Q_i^t$ and $Q_i^{\star}$ we have
	\begin{equation}
		(1-2\eta)\mathbb{V}_e^t(Q_i^t)-(1-\eta)\mathbb{V}_e^t(Q_i^{\star})\leq \langle Q_i^t-Q_i^{\star}, \nabla_{Q_i^t}\mathbb{V}_e^t(Q_i^t)\rangle + \eta\langle Q_i^{\star}, \nabla_{Q_i^t}\mathbb{V}_e^t(Q_i^t)\rangle
	\end{equation}	
\end{lemma}
\begin{proof}
The function $\mathbb{V}_e^t(Q)$ is convex with respect to $Q$, hence for any $Q_i^t$ and $Q_i^{\star}$ we have
\begin{equation}
	\mathbb{V}_e^t(Q_i^t)-\mathbb{V}_e^t(Q_i^{\star})\leq \langle Q_i^{t}-Q_i^{\star}, \nabla_{Q_i^t}\mathbb{V}_e^t(Q_i^t)\rangle.
\end{equation}

Multiplying both sides of this inequality by $1-\eta$, we have
\begin{align}
	&(1-\eta)\mathbb{V}_e^t(Q_i^t)-(1-\eta)\mathbb{V}_e^t(Q_i^{\star}) \\
	&\leq \langle Q_i^{t}-Q_i^{\star}, \nabla_{Q_i^t}\mathbb{V}_e^t(Q_i^t)\rangle - \eta\langle Q_i^{t}-Q_i^{\star}, \nabla_{Q_i^t}\mathbb{V}_e^t(Q_i^t)\rangle.
\end{align}

In the following, we prove this Lemma in our two bandit settings: \textit{adversary MAB setting} and \textit{adversary MAB with multiple plays setting}.

In \textit{adversary MAB setting}, we have
\begin{align}
	\langle Q_i^t, \nabla_{Q_i^t}\mathbb{V}_e^t(Q_i^t)\rangle &= -\sum_{j\in\mathcal{N}_i}q_{ij}(t)\frac{\alpha_{ij}^2}{k\cdot q_{ij}(t)^2}\|h_j(t)\|^2 \\
	&= -\mathbb{V}_e^t(Q_i^t)
\end{align}

In \textit{adversary MAB with multiple plays setting}, we use the approximated effective variance $\sum_{j_s\in\mathcal{N}_i}\frac{\alpha_{ij_s}}{q_{ij_s}}\|h_{j_s}\|^2$ derived in Proposition~\ref{proposition:mp_var_bound}. For notational simplicity, we denote the approximated effective variance as $\mathbb{V}_e$ in the following. We have
\begin{align}
	\langle Q_i^t, \nabla_{Q_i^t}\mathbb{V}_e^t(Q_i^t)\rangle &= -\sum_{S_i\subset\mathcal{N}_i}Q_{i,S_i}^t\sum_{j_s\in S_i}\frac{\alpha_{ij_s}}{q_{ij_s}(t)^2}\|h_{j_s}\|^2 \label{eq:effective_variance_derivative} \\
	&= -\sum_{j_s\in\mathcal{N}_i}\frac{\alpha_{ij_s}}{q_{ij_s}(t)^2}\|h_{j_s}\|^2\sum_{S_i:j_s\in S_i}Q_{i,S_i}^t \\
	&= -\sum_{j_s\in\mathcal{N}_i}\frac{\alpha_{ij_s}}{q_{ij_s}(t)}\|h_{j_s}\|^2 \\
	&= -\mathbb{V}_e^t(Q_i^t).
\end{align}
The equation~\eqref{eq:effective_variance_derivative} holds because of Proposition~\ref{proposition:mp_derivative}.

At last, we conclude the proof
\begin{align}
	&(1-\eta)\mathbb{V}_e^t(Q_i^t)-(1-\eta)\mathbb{V}_e^t(Q_i^{\star}) \\
	&\leq \langle Q_i^{t}-Q_i^{\star}, \nabla_{Q_i^t}\mathbb{V}_e^t(Q_i^t)\rangle - \eta\langle Q_i^{t}-Q_i^{\star}, \nabla_{Q_i^t}\mathbb{V}_e^t(Q_i^t)\rangle \\
	&= \langle Q_i^{t}-Q_i^{\star}, \nabla_{Q_i^t}\mathbb{V}_e^t(Q_i^t)\rangle + \eta\langle Q_i^{\star}, \nabla_{Q_i^t}\mathbb{V}_e^t(Q_i^t)\rangle + \eta\mathbb{V}_e^t(Q_i^t).
\end{align}
\end{proof}

\begin{theorem}\label{theorem:bs}
	Using Algorithm~\ref{alg:train_gnn} with $\eta=0.4$ and $\delta=\sqrt{(1-\eta)\eta^4 k^5 \ln(n/k)/(T n^4)}$ to minimize effective variance with respect to $\{Q_i^t\}_{1\leq t\leq T}$, we have 
	\begin{equation}
		\sum_{t=1}^T\mathbb{V}_e^t(Q_i^t) \leq 3\sum_{t=1}^T\mathbb{V}_e^t(Q_i^{\star}) + 10\sqrt{\frac{Tn^4\ln(n/k)}{k^3}}
	\end{equation}
	where $T\geq \ln(n/k)n^2(1-\eta)/(k\eta^2)$ and $n=\left|\mathcal{N}_i\right|$.
\end{theorem}

\begin{proof}
First we explain why condition $T\geq \ln(n/k)n^2(1-\eta)/(k\eta^2)$ ensures that $\delta\hat{r}_{ij}(t)\leq 1$,
{\footnotesize
\begin{align}
	\delta\hat{r}_{ij}(t) &= \sqrt{\frac{(1-\eta)\eta^4 k^5 \ln(n/k)}{T n^4}}\cdot\frac{\alpha_{ij}(t)}{q_{ij}^3(t)}\|h_{j}(t)\|^2 \\
	&\leq \sqrt{\frac{(1-\eta)\eta^4k^5\ln(n/k)}{Tn^4}}\cdot\frac{n^3}{k^3\eta^3} \label{eq:apx1} \\
	&\leq 1
\end{align}
}
Assuming $\|h_{j}(t)\|\leq 1$, inequality~\eqref{eq:apx1} holds because $\alpha_{ij}(t)\leq 1$ and $q_{ij}(t)\geq k\eta/n$. Then replace $T$ by the condition, we get $\delta\hat{r}_{ij}(t)\leq 1$.

Let $W_i(t)$, $W_i^{\prime}(t)$ denote $\sum_{j\in\mathcal{N}_i}w_{ij}(t)$, $\sum_{j\in\mathcal{N}_i}w_{ij}^{\prime}(t)$ respectively. Then for any $t=1,2,...,T$,
{\footnotesize
\begin{align}
\frac{W_i(t+1)}{W_i(t)} &= \sum_{j\in\mathcal{N}_i\backslash U_i^t}\frac{w_{ij}(t+1)}{W_{i}(t)} + \sum_{j\in U_i^t}\frac{w_{ij}(t+1)}{W_{i}(t)} \\ 
&= \sum_{j\in\mathcal{N}_i\backslash U_i^t}\frac{w_{ij}(t)}{W_{i}(t)}\cdot\exp(\delta\hat{r}_{ij}(t)) + \sum_{j\in U_i^t}\frac{w_{ij}(t)}{W_{i}(t)} \\
&\leq \sum_{j\in\mathcal{N}_i\backslash U_i^t}\frac{w_{ij}(t)}{W_{i}(t)}\left[1 + \delta\hat{r}_{ij}(t) + \left(\delta\hat{r}_{ij}(t)\right)^2\right] + \sum_{j\in U_i^t}\frac{w_{ij}(t)}{W_{i}(t)} \label{eq:apx2} \\
&= 1 + \frac{W_i^{\prime}(t)}{W_i(t)}\sum_{j\in\mathcal{N}_i\backslash U_i^t}\frac{w_{ij}(t)}{W_i^{\prime}(t)}\left[\delta\hat{r}_{ij}(t) +  \left(\delta\hat{r}_{ij}(t)\right)^2\right] \\
&= 1 + \frac{W_i^{\prime}(t)}{W_i(t)}\sum_{j\in\mathcal{N}_i\backslash U_i^t}\frac{q_{ij}(t)/k - \eta/n}{1-\eta}\left[\delta\hat{r}_{ij}(t) +  \left(\delta\hat{r}_{ij}(t)\right)^2\right] \label{eq:apx14} \\
&\leq 1 + \frac{\delta}{k(1-\eta)}\sum_{j\in\mathcal{N}_i\backslash U_i^t}q_{ij}(t)\hat{r}_{ij}(t) + \frac{\delta^2}{k(1-\eta)}\sum_{j\in\mathcal{N}_i\backslash U_i^t}q_{ij}(t)\hat{r}_{ij}^2(t) \label{eq:apx3}
\end{align}
}

Inequality~\eqref{eq:apx2} uses $e^a\leq 1+a+a^2$ for $a\leq 1$. Equality~\eqref{eq:apx14} holds because of update equation of $q_{ij}(t)$ defined in EXP3.M. Inequality~\eqref{eq:apx3} holds because $\frac{W_i^{\prime}(t)}{W_i(t)}\leq 1$. Since $1+x\leq e^x$ for $x\geq 0$, we have
{\footnotesize
\begin{equation}
	\ln\frac{W_i(t+1)}{W_i(t)} \leq \frac{\delta}{k(1-\eta)}\sum_{j\in\mathcal{N}_i\backslash U_i^t}q_{ij}(t)\hat{r}_{ij}(t) + \frac{\delta^2}{k(1-\eta)}\sum_{j\in\mathcal{N}_i\backslash U_i^t}q_{ij}(t)\hat{r}_{ij}^2(t)
\end{equation}
}

If we sum, for $1\leq t\leq T$, we get the following telescopic sum
{\footnotesize
\begin{align}
	\ln\frac{W_i(T+1)}{W_i(1)} &= \sum_{t=1}^{T}\ln\frac{W_i(t+1)}{W_i(t)} \\
	&\leq \frac{\delta}{k(1-\eta)}\sum_{t=1}^T\sum_{j\in\mathcal{N}_i\backslash U_i^t}q_{ij}(t)\hat{r}_{ij}(t) + \frac{\delta^2}{k(1-\eta)}\sum_{t=1}^T\sum_{j\in\mathcal{N}_i\backslash U_i^t}q_{ij}(t)\hat{r}_{ij}^2(t) \\
	&\leq \frac{\delta}{k(1-\eta)}\sum_{t=1}^T\sum_{j\in\mathcal{N}_i\backslash U_i^t}q_{ij}(t)\hat{r}_{ij}(t) + \frac{\delta^2}{k(1-\eta)}\sum_{t=1}^T\sum_{j\in\mathcal{N}_i}q_{ij}(t)\hat{r}_{ij}^2(t) \label{eq:apx5}
\end{align}
}

On the other hand, for any subset $S$ containing k elements,
{\footnotesize
\begin{align}
	\ln\frac{W_i(T+1)}{W_i(1)}&\geq \ln\frac{\sum_{j\in S}w_{ij}(T+1)}{W_i(1)} \\
	&\geq \frac{\sum_{j\in S}\ln w_{ij}(T+1)}{k} - \ln\frac{n}{k} \label{eq:apx6} \\
	&\geq \frac{\delta}{k}\sum_{j\in S}\sum_{t:j\notin U_i^t}\hat{r}_{ij}(t) - \ln\frac{n}{k} \label{eq:apx7}
\end{align}
}
The inequality~\eqref{eq:apx6} uses the fact that
{\footnotesize
\begin{equation*}
	\sum_{j\in S}w_{ij}(T+1)\geq k(\prod_{j\in S}w_{ij}(T+1))^{1/k}
\end{equation*}
}

The equation~\eqref{eq:apx7} uses the fact that
{\footnotesize
\begin{equation*}
	w_{ij}(T+1)=\exp(\delta\sum_{t:j\notin U_i^t}\hat{r}_{ij}(t))
\end{equation*}
}

From~\eqref{eq:apx5} and ~\eqref{eq:apx7}, we get
{\footnotesize
\begin{equation}
	\frac{\delta}{k}\sum_{j\in S}\sum_{t:j\notin U_i^t}\hat{r}_{ij}(t) - \ln\frac{n}{k} \leq \frac{\delta}{k(1-\eta)}\sum_{t=1}^T\sum_{j\in\mathcal{N}_i\backslash U_i^t}q_{ij}(t)\hat{r}_{ij}(t) + \frac{\delta^2}{k(1-\eta)}\sum_{t=1}^T\sum_{j\in\mathcal{N}_i\backslash U_i^t}q_{ij}(t)\hat{r}_{ij}^2(t) \label{eq:apx13}
\end{equation}
}

And we have the following inequality
{\footnotesize
\begin{align}
	\frac{\delta}{k}\sum_{j\in S}\sum_{t:j\in U_i^t}r_{ij}(t) &= \frac{\delta}{k}\sum_{j\in S}\sum_{t:j\in U_i^t}q_{ij}(t)\hat{r}_{ij}(t) \label{eq:apx11} \\
		&\leq \frac{\delta}{k(1-\eta)}\sum_{t=1}^T\sum_{j\in U_i^t}q_{ij}(t)\hat{r}_{ij}(t) \label{eq:apx12}
\end{align}
}
The equality~\eqref{eq:apx11} holds beacuse $r_{ij}(t)=q_{ij}\hat{r}_{ij}(t)$ when $j\in S_i^t$ and $U_i^t\subseteq S_i^t$ bacause $q_{ij}^t=1$ for all $j\in U_i^t$.

Then add inequality~\eqref{eq:apx12} in ~\eqref{eq:apx13} we have
{\footnotesize
\begin{align}\label{eq:apx8}
&	\frac{\delta}{k}\sum_{j\in S}\sum_{t:j\in U_i^t}r_{ij}(t) + \frac{\delta}{k}\sum_{j\in S}\sum_{t:j\notin U_i^t}\hat{r}_{ij}(t) - \ln\frac{n}{k} \\
& \leq \frac{\delta}{k(1-\eta)}\sum_{t=1}^T\sum_{j\in\mathcal{N}_i}q_{ij}(t)\hat{r}_{ij}(t) + \frac{\delta^2}{k(1-\eta)}\sum_{t=1}^T\sum_{j\in\mathcal{N}_i}q_{ij}(t)\hat{r}_{ij}^2(t) 
\end{align}
}
Given $q_{ij}(t)$ we have $\mathbb{E}[\hat{r}_{ij}^2(t)]=r_{ij}^2(t)/q_{ij}(t)$, hence, taking expectation of~\eqref{eq:apx8} yields that
{\footnotesize
\begin{equation}
	\frac{\delta}{k}\sum_{t=1}^T\sum_{j\in S}r_{ij}(t) - \ln\frac{n}{k}\leq \frac{\delta}{k(1-\eta)}\sum_{t=1}^T\sum_{j\in\mathcal{N}_i}q_{ij}(t) r_{ij}(t) + \frac{\delta^2}{k(1-\eta)}\sum_{t=1}^T\sum_{j\in\mathcal{N}_i}r_{ij}^2(t) \label{eq:apx9}
\end{equation}
}
By multiplying~\eqref{eq:apx9} by $Q_{i,S}^{\star}$ and summing over $S$, we get
{\footnotesize
\begin{equation}
	\frac{\delta}{k}\sum_{t=1}^T\sum_{S\subset\mathcal{N}_i} Q_{i,S}^{\star}\sum_{j\in S}r_{ij}(t) - \ln\frac{n}{k}\leq \frac{\delta}{k(1-\eta)}\sum_{t=1}^T\sum_{j\in\mathcal{N}_i}q_{ij}(t) r_{ij}(t) + \frac{\delta^2}{k(1-\eta)}\sum_{t=1}^T\sum_{j\in\mathcal{N}_i}r_{ij}^2(t) \label{eq:apx10}
\end{equation}
}
As
{\footnotesize
\begin{align}
	\sum_{j\in\mathcal{N}_i}q_{ij}(t) r_{ij}(t)&=\sum_{j\in\mathcal{N}_i}\sum_{S_i:j\in S_i}Q_{i,S_i}^t r_{ij}(t) \\
	&= \sum_{S_i\subset\mathcal{N}_i}Q_{i,S_i}^t\sum_{j\in S_i}r_{ij}(t) \\
	&= -\sum_{S_i\subset\mathcal{N}_i}Q_{i,S_i}^t\nabla_{Q_{i,S_i}^t}\mathbb{V}_e^t(Q_{i,S_i}^t) \\
	&= -\langle Q_{i}^t,\nabla_{Q_i^t}\mathbb{V}_e^t(Q_{i}^t)\rangle \label{eq:apx15}
\end{align}
}

By plugging \eqref{eq:apx15} in~\eqref{eq:apx10} and rearranging it, we find
{\footnotesize
\begin{align}
	&\sum_{t=1}^T\langle Q_{i}^t-Q_{i}^{\star}, \nabla_{Q_i^t}\mathbb{V}_e^t(Q_{i}^t)\rangle + \eta\sum_{t=1}^T\langle Q_i^{\star}, \nabla_{Q_i^t}\mathbb{V}_e^t(Q_i^{t})\rangle \\\nonumber
& \leq \delta\sum_{t=1}^T\sum_{j\in\mathcal{N}_i}r_{ij}^2(t) + \frac{(1-\eta)k}{\delta}\ln (n/k)
\end{align}
}
Using Lemma~\ref{lemma:1}, we have
{\footnotesize
\begin{equation}
	(1-2\eta)\sum_{t=1}^T\mathbb{V}_e^t(Q_i^t)-(1-\eta)\sum_{t=1}^T\mathbb{V}_e^t(Q_i^{\star})\leq \delta\sum_{t=1}^T\sum_{j\in\mathcal{N}_i}r_{ij}^2(t) + \frac{(1-\eta)k}{\delta}\ln (n/k)
\end{equation}
}
Finally, we know that 
{\footnotesize
\begin{align}
	\sum_{j\in \mathcal{N}_i}r_{ij}^2(t)&=\sum_{j\in\mathcal{N}_i}\frac{\alpha_{ij}(t)^2}{q_{ij}(t)^4} \\
	&\leq \sum_{j\in\mathcal{N}_i}\alpha_{ij}(t)\frac{n^4}{k^4\eta^4} \; (because\, q_{ij}(t)\geq k\eta/n) \\
	&= \frac{n^4}{k^4\eta^4}
\end{align}
}
By setting $\eta=0.4$ and $\delta=\sqrt{(1-\eta)\eta^4 k^5 \ln(n/k)/(T n^4)}$, we get the upper bound.

\end{proof}

\section{Experiments}

\subsection{Convergences}\label{appendix:convergences}
We show the convergences on validation in terms of timing (seconds)
in Figure~\ref{fig:timing_gcn} and Figure~\ref{fig:timing_gat}. Basically, our algorithms converge
to much better results in nearly same duration compared with other
``layer sampling'' approaches.

Note that we cannot complete the training of AS-GAT on Reddit
because of memory issues.

\begin{figure*}[h]
\includegraphics[width=0.33\textwidth,height=0.33\textwidth]{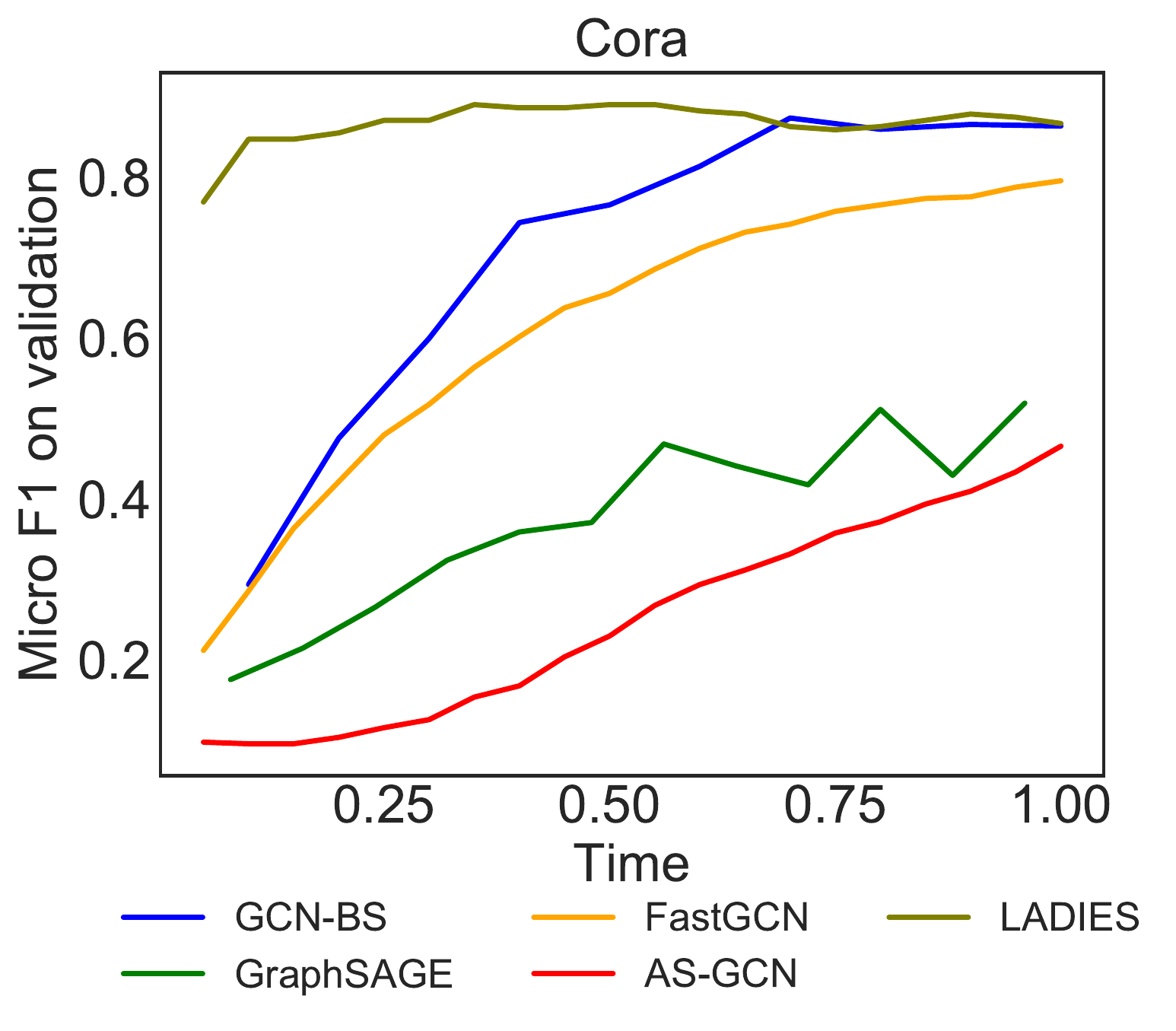}
\includegraphics[width=0.33\textwidth,height=0.33\textwidth]{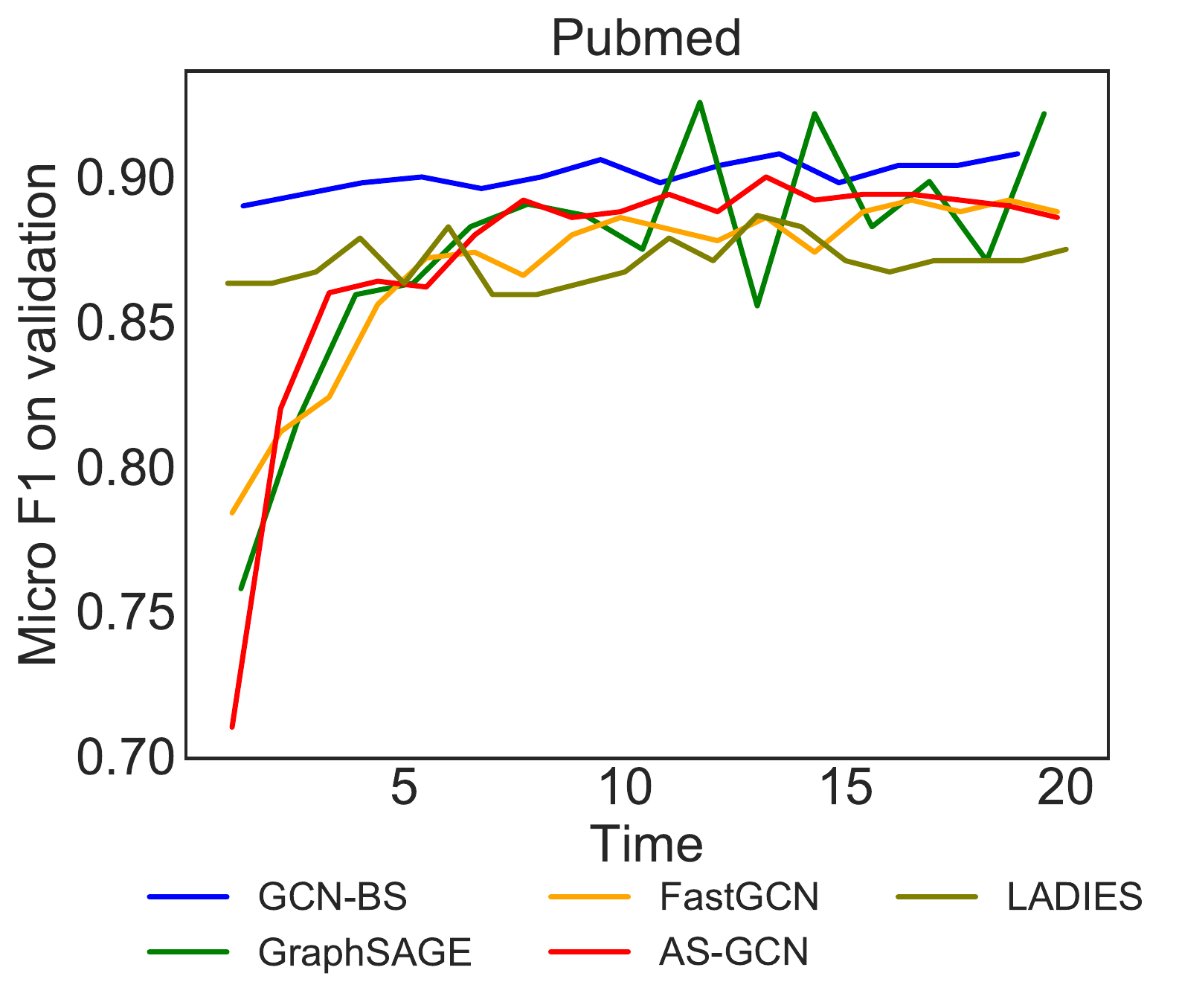}
\includegraphics[width=0.33\textwidth,height=0.33\textwidth]{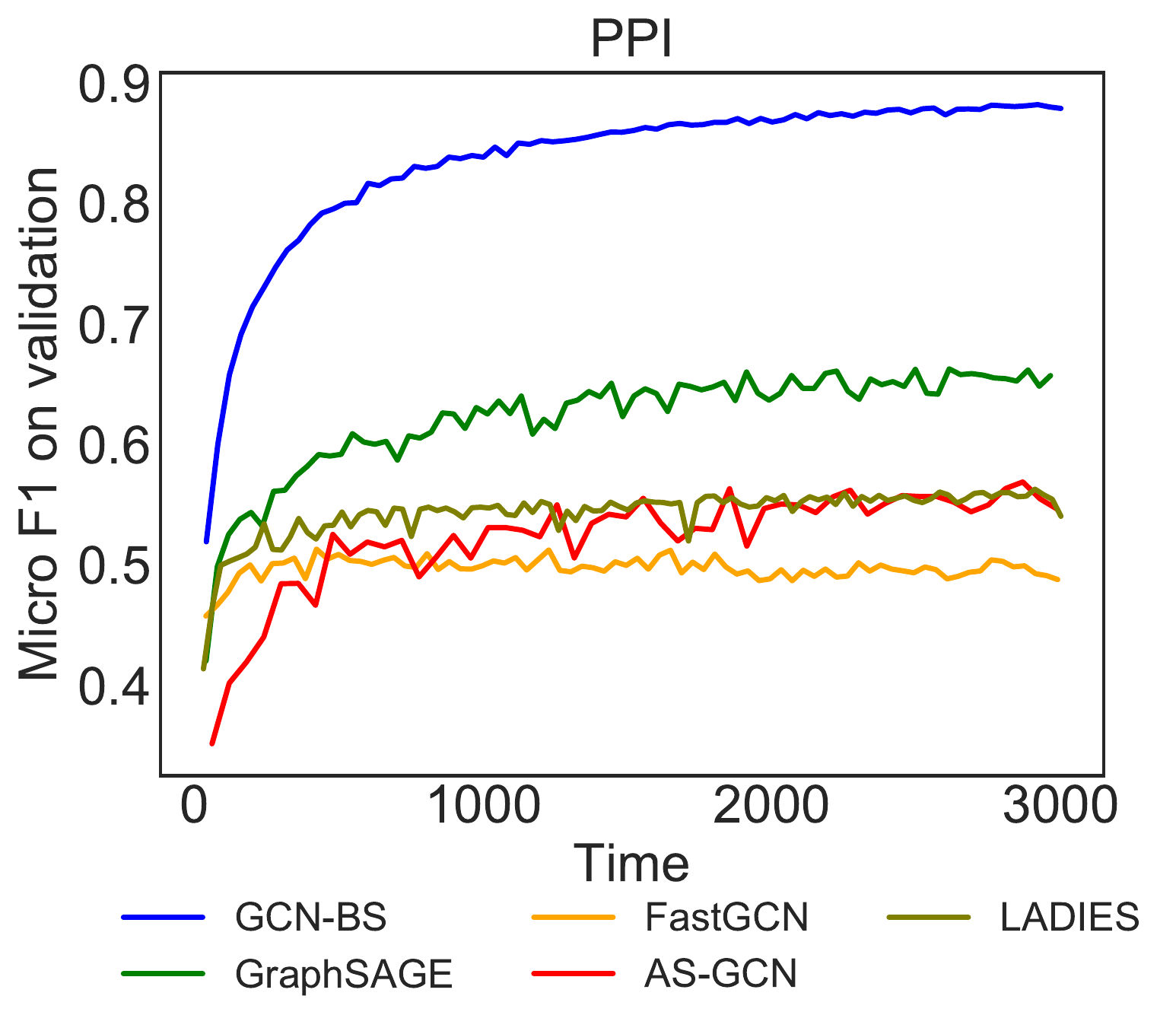}
\includegraphics[width=0.33\textwidth,height=0.33\textwidth]{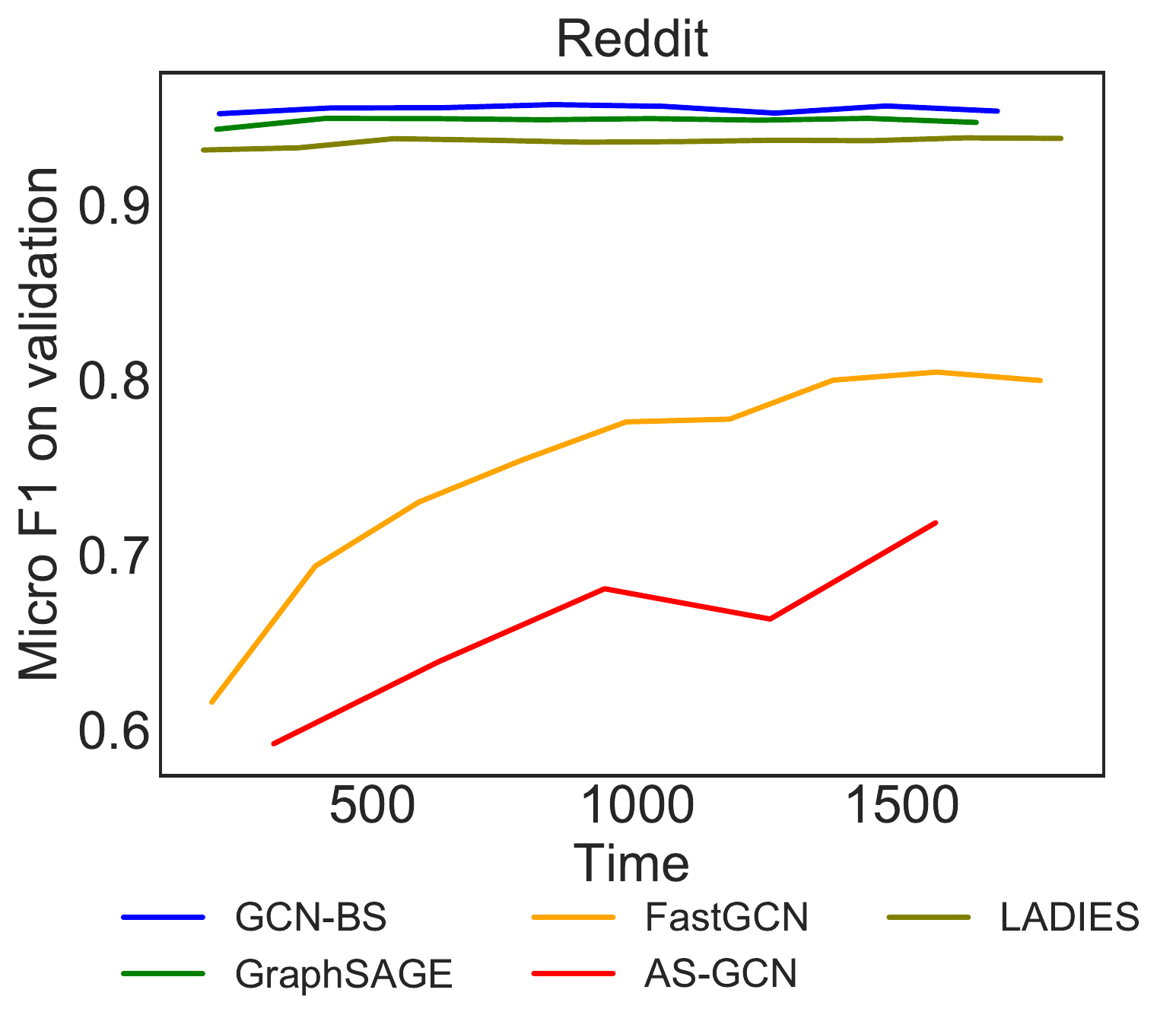}
\includegraphics[width=0.33\textwidth,height=0.33\textwidth]{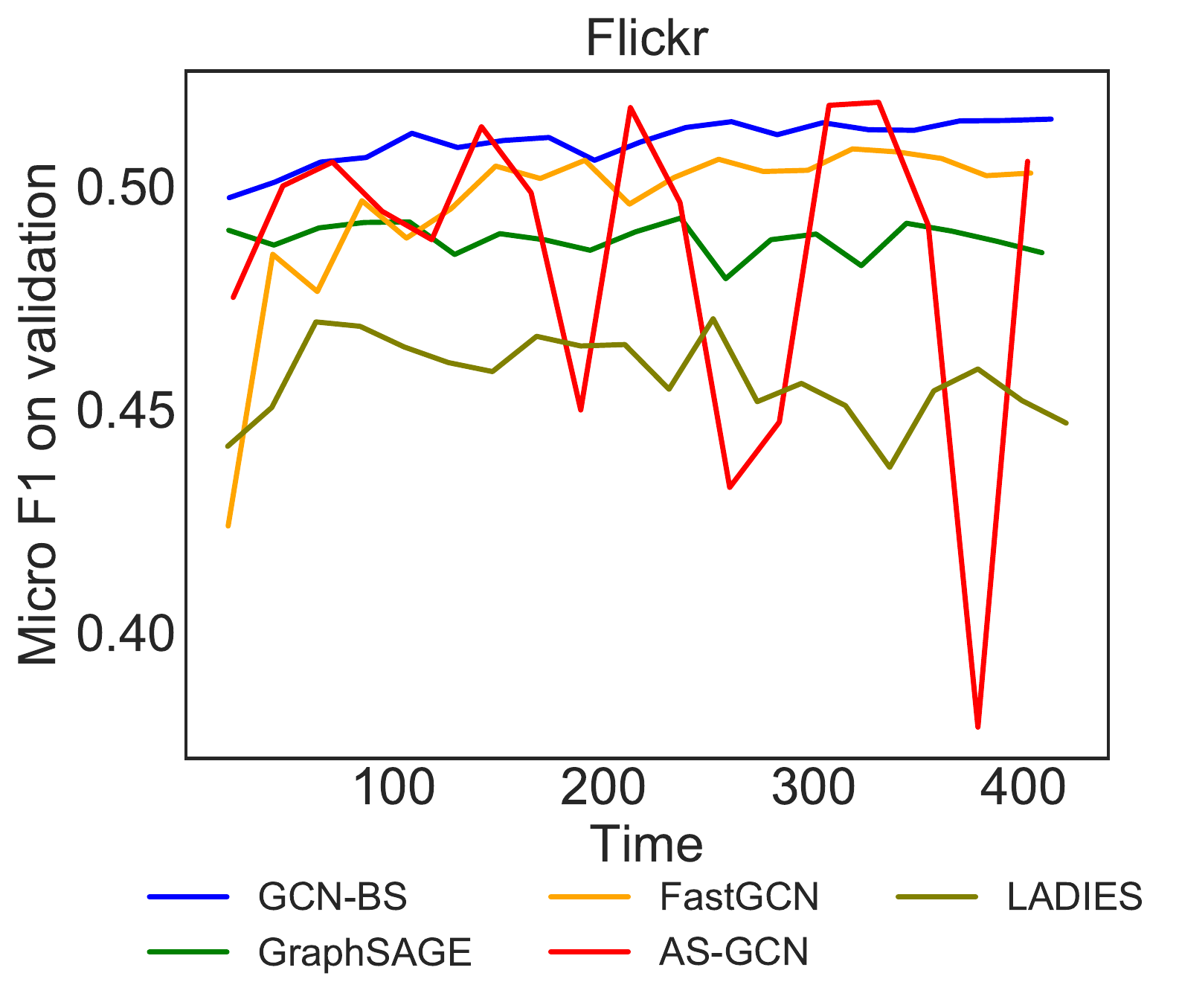}
\caption{The convergence in timing (seconds) on GCNs.}
\label{fig:timing_gcn}
\end{figure*}

\begin{figure*}[h]
\includegraphics[width=0.24\textwidth,height=0.24\textwidth]{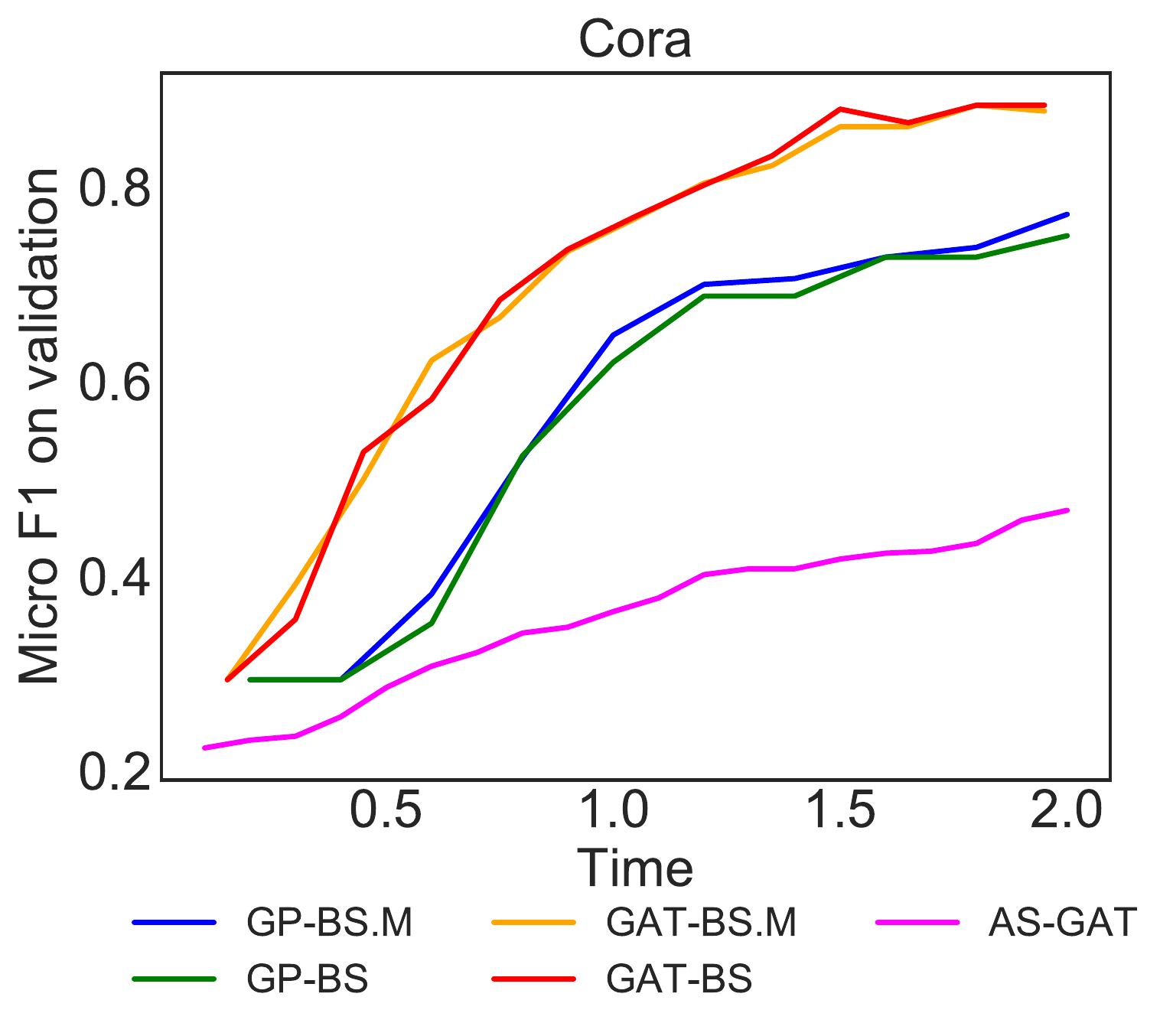}
\includegraphics[width=0.24\textwidth,height=0.24\textwidth]{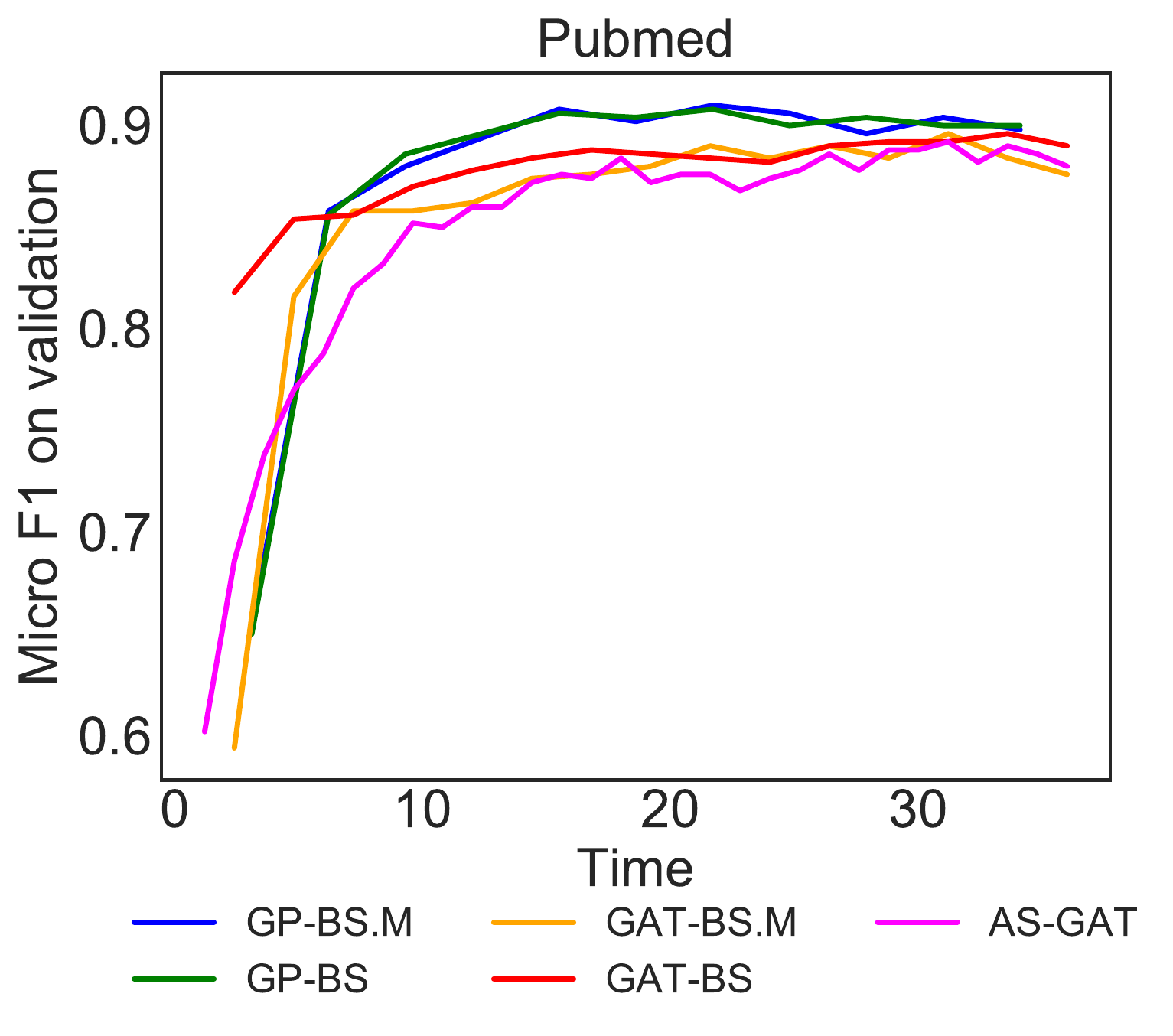}
\includegraphics[width=0.24\textwidth,height=0.24\textwidth]{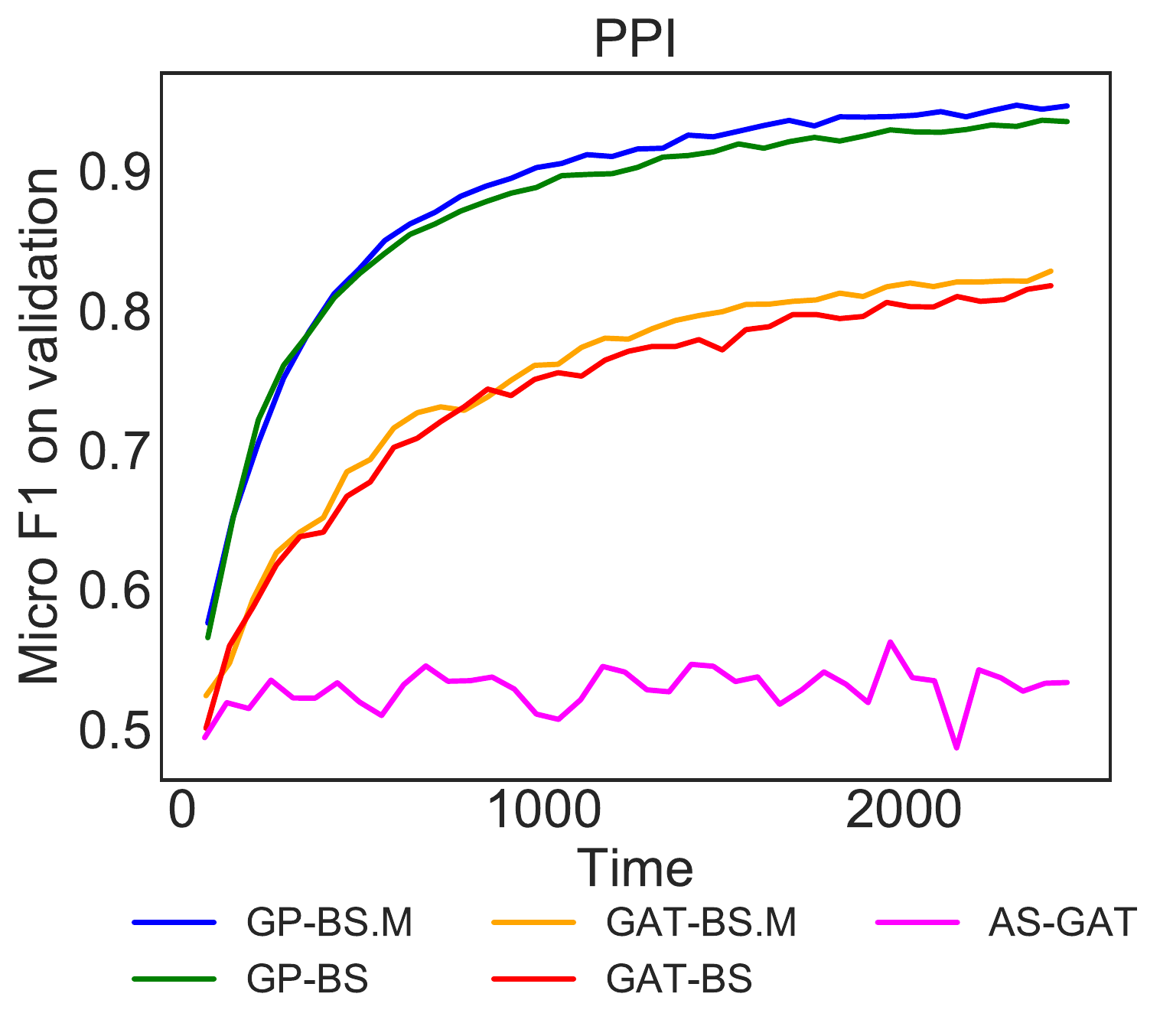}
\includegraphics[width=0.24\textwidth,height=0.24\textwidth]{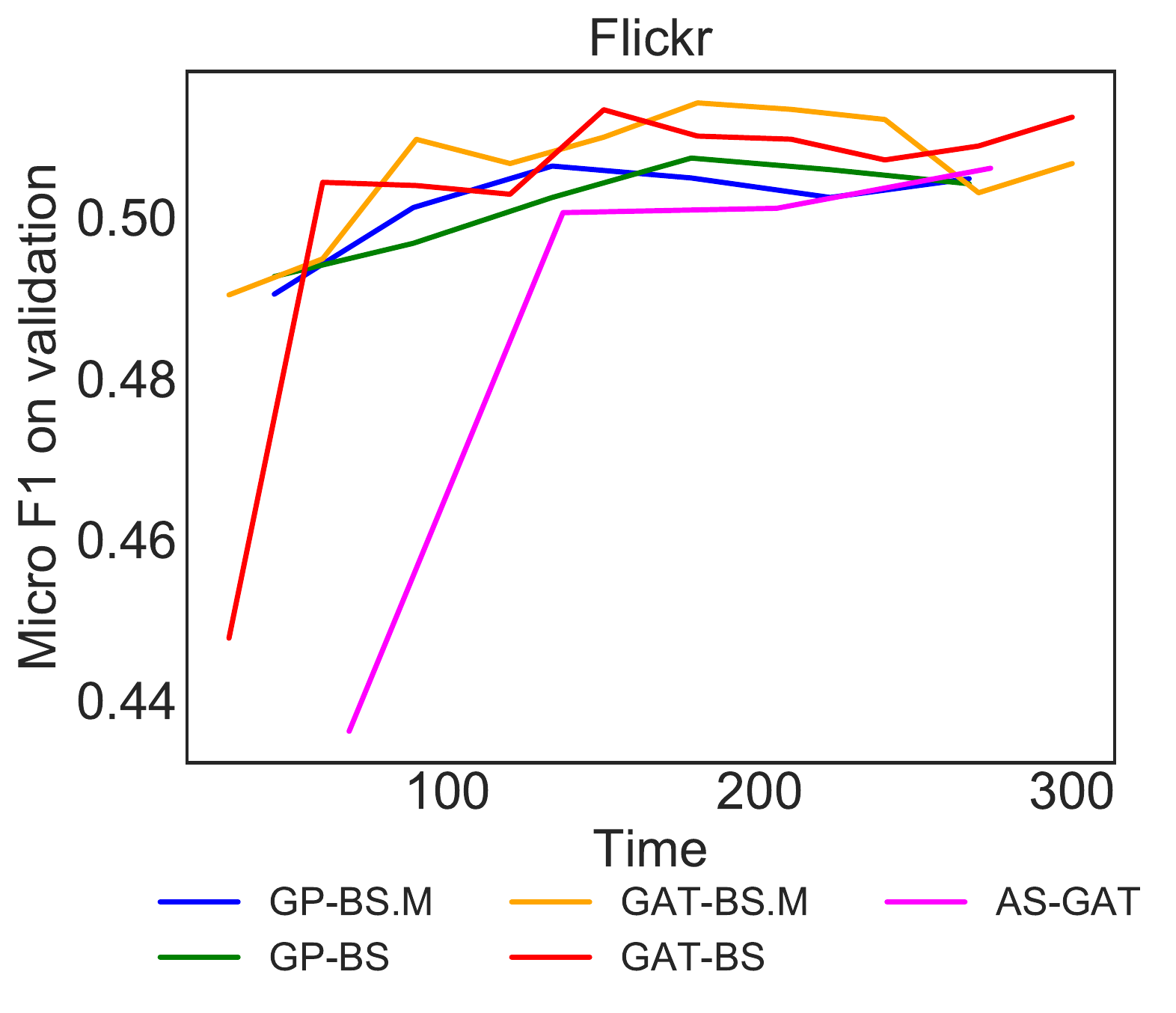}
\caption{The convergence in timing (seconds) on attentive GNNs.}
\label{fig:timing_gat}
\end{figure*}
\subsection{Discussions on Timings between Layer Sampling and Graph Sampling Paradigms}\label{appendix:layer_vs_graph}

Note that the comparisons of 
timing between ``graph sampling'' and ``layer sampling'' 
paradigms have been studied recently in~\cite{chiang2019cluster,zeng2019graphsaint}.
As a result, we do not compare the timing with ``graph sampling'' approaches.
Under certain conditions, the graph sampling approaches 
should be faster than layer sampling approaches. That is, graph sampling approaches
are designed for graph data that all vertices have labels.
Under such condition, the floating point operations analyzed in~\cite{chiang2019cluster}
are maximally utilized compared with the ``layer sampling'' paradigm.
However, in practice, there are large amount of graph data with labels
only on some types of vertices, such as the graphs in~\cite{liu2018heterogeneous}.
``Graph sampling'' approaches
are not applicable to cases where only partial vertices have labels.
To summarize, the
``layer sampling'' approaches are more flexible
and general compared with ``graph sampling''
approaches in many cases.

\subsection{Results on OGB}
We report our results on 
OGB protein dataset~\cite{hu2020open}. We set the learning rate as 1e-3, batch size as 256, the dimension of hidden embeddings as 64, sample size as 10 and epochs as 200. We save the model based on the best results on validation and report results on testing data. We run the experiment 10 times with different random seeds to compute the average and standard deviation of the results. Our result of GP-BS on protein dataset performs the best\footnote{Please refer to \url{https://ogb.stanford.edu/docs/leader_nodeprop/}.}
until we submitted this paper. Please find our implementations at \url{https://github.com/xavierzw/ogb-geniepath-bs}.

\setlength{\tabcolsep}{1pt}
\begin{table}[H]
  \centering
  \begin{tabular}{cccc}
  	\toprule
    Dateset & Mean & Std & \#experiments \\
    \midrule
    ogbn-proteins&  $0.78253$ & $0.00352$ & $10$ \\
  \bottomrule
\end{tabular}
\end{table}

\end{document}